\documentclass{article}

\usepackage[nonatbib,preprint]{neurips_2019}

\usepackage[numbers]{natbib}

\usepackage{microtype}
\usepackage{graphicx}
\usepackage{subcaption}
\usepackage{booktabs} 
\usepackage{hyperref}

\usepackage{amsmath}
\usepackage{amssymb}
\usepackage{mathtools}
\usepackage{amsthm}
\usepackage{enumitem}
\usepackage{multirow}

\usepackage[capitalize,noabbrev]{cleveref}

\theoremstyle{plain}
\newtheorem{theorem}{Theorem}[section]
\newtheorem{proposition}[theorem]{Proposition}
\newtheorem{lemma}[theorem]{Lemma}
\newtheorem{corollary}[theorem]{Corollary}
\theoremstyle{definition}

\theoremstyle{remark}

\DeclareMathOperator*{\argmin}{arg\,min}

\providecommand{\Paren}[1]{\ensuremath{\left( #1 \right)}}

\providecommand{\cL}{\mathcal{L}}
\def\tr{^\top}
\def\cpm{\!\pm\!}

\providecommand{\norm}[1]{\lVert#1\rVert}

\newcounter{tablerownumbers}

\newcommand{\eps}{\varepsilon}

\def\m#1{\ensuremath{\mathtt{#1}}}

\def\mI{{\m I}}

\begin{document}

\title{Two Tales of Single-Phase Contrastive Hebbian Learning}

\author{%
  Rasmus Kj\ae{}r H\o{}ier and Christopher Zach  \\
  Chalmers University of Technology\\
  Gothenburg, Sweden \\
  \texttt{\{hier,zach\}@chalmers.se}
}

\maketitle

\begin{abstract}
  The search for "biologically plausible" learning algorithms
  has converged on the idea of representing gradients as activity
  differences. However, most approaches require a high degree of
  synchronization (distinct phases during learning) and introduce substantial
  computational overhead, which raises doubts regarding their biological
  plausibility as well as their potential utility for neuromorphic
  computing. Furthermore, they commonly rely on applying infinitesimal
  perturbations (nudges) to output units, which is impractical in noisy
  environments. Recently it has been shown that by modelling artificial
  neurons as dyads with two oppositely nudged compartments, it is possible for
  a fully local learning algorithm named ``dual propagation'' to bridge the performance gap to
  backpropagation, without requiring separate learning phases or infinitesimal nudging.
  However, the algorithm has the drawback that its numerical stability 
  relies on symmetric nudging, which may be restrictive in biological
  and analog implementations.
  In this work we first provide a solid foundation for the objective underlying the dual propagation method, which also reveals a surprising connection with adversarial robustness.
  Second, we demonstrate how dual propagation is related to a particular adjoint state method, which is stable regardless of asymmetric nudging.
\end{abstract}

\section{Introduction}
\label{intro}

Credit assignment using fully local alternatives to back-propagation is interesting both as potential models of biological learning as well as for their applicability for energy efficient analog neuromorphic computing~\citep{kendall2020building, yi2022memristorcrossbars}.
A pervasive idea in this field is the idea of representing error signals via activity differences, referred to as NGRAD (Neural Gradient Representation by Activity Differences) approaches~\citep{lillicrap2020NGRAD}.
However, a commonly overlooked issue in NGRAD approaches is the requirement of applying infinitesimal perturbations (or \emph{nudging}) to output neurons in order to propagate error information through the network. This is problematic as analog and biological neural networks are inherently noisy, potentially causing the error signal to vanish if insufficient nudging is applied.
In many local learning methods the output units are \emph{positively} nudged to reduce a target loss, but utilizing additional negative nudging (output units increase a target loss) can be beneficial to improve accuracy (e.g.~\citep{laborieux2021scaling}).

The vanishing error signal problem is addressed by coupled learning~\citep{stern2021supervised}, which proposes to replace the clamped output units of contrastive Hebbian learning with a convex combination of the label and the free phase outputs.
Unfortunately, coupled learning has been shown to perform worse than equilibrium propagation on CIFAR10 and CIFAR100~\citep{scellier2023energybased}, and it does not necessarily approximate gradient descent on the output loss function~\citep{stern2021supervised}.
Holomorphic equilibrium propagation~\citep{laborieux2022holomorphic,laborieux2023improvingholomorphic} mitigates the need for infinitesimal nudging required in standard equilibrium propagation~\citep{scellier2017} at the cost of introducing complex-valued parameters.
Whether this is a suitable approach for either biological or analog neural networks is an open question.
Dual propagation~\cite{hoier2023conf} (DP), an algorithm similar in spirit to contrastive Hebbian learning, equilibrium propagation and coupled learning, is compatible with non-infinitesimal nudging by default. This method infers two sets of oppositely nudged and mutually tethered states simultaneously. 
However, utilization of symmetric nudging is a necessary condition for the convergence of its inference step.

\paragraph{Contributions} DP is compatible with strong feedback and only requires a single inference phase, which are appealing features with regards to biological plausibility and potential applications to analog neuromorphic computing. However, the lack of convergence guarantees in the case of asymmetric nudging is clearly unsettling as exact symmetry is hard to realize outside of digital computers.
Further,---unlike digital computers---neuromorphic, analog or otherwise highly distributed computing hardware typically performs continuous computations and runs asynchronously. Consequently, numerical stability of an energy-based inference and learning method is of essential importance.
For this reason we derive an improved variant of dual propagation, which overcomes this strict symmetry requirement. In summary the contributions of this work are:
\begin{itemize}[nosep,left=0pt]
    \item A derivation of DP based on repeated relaxations of the optimal value reformulation~\cite{outrata1988note,dempe2013bilevel}.
    \item A new Lagrangian based derivation of dual propagation, which recovers the original dual propagation algorithm in the case of symmetric nudging, but leads to a slightly altered (and more robust) method in the case of asymmetric nudging.
    \item We experimentally investigate the robustness of these algorithms to asymmetric nudging and strong feedback, and we further demonstrate the impact of asymmetric nudging on the estimated Lipschitz constant.
\end{itemize}

\section{Related Work}
\paragraph{CHL, EP and lifted networks}
In contrastive Hebbian learning (CHL)~\citep{movellan1991, xie2003} and equilibrium propagation (EP)~\citep{scellier2017} neuronal activations are found via an energy minimization procedure. Inference is carried out twice, once with and once without injecting label information at the output layer. While CHL clamps the output units to the true targets, EP applies nudging towards a lower loss. The difference between the activity in each of these two inference phases is used to represent neuronal error vectors. 
To speed up convergence and ensure that inferred states represent the same local energy basin, this is typically (but not always) done sequentially, e.g.\ the second inference phase is initialized with the solution found in the first phase.
A better gradient estimate can be obtained by introducing an additional oppositely nudged inference phase, yielding a three-phased algorithm~\cite{laborieux2021scaling}.

\begin{figure}
     \centering
     \begin{subfigure}[b]{0.42\textwidth}
         \centering
         \includegraphics[width=\textwidth]{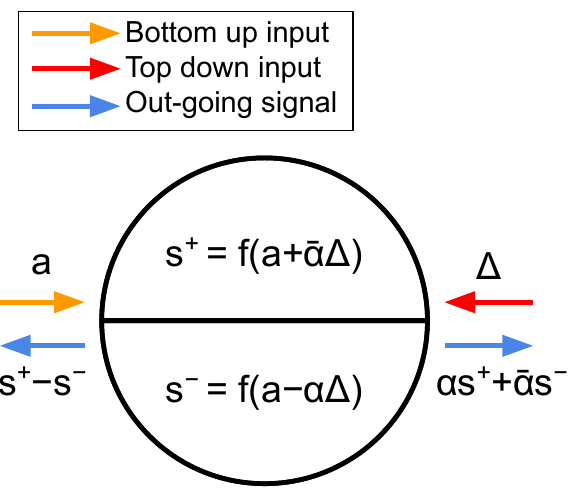}
         \caption{Sketch of a dyadic DP neuron.}
         \label{fig:dyadicneuron}
     \end{subfigure}~~~~
     \begin{subfigure}[b]{0.30\textwidth}
         \centering
         \includegraphics[width=0.9\textwidth]{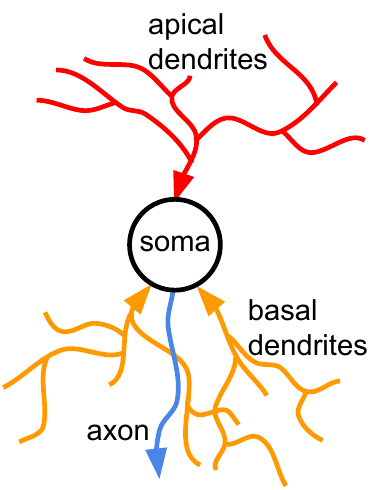}
         \caption{Sketch of a pyramidal neuron.}
         \label{fig:pyramidalneuron}
     \end{subfigure}
        \caption{\textbf{(a)} Illustration of a dyadic neuron (note that all quantities are scalar). The two internal states, $s^+$ and $s^-$, receive the same bottom-up input $a$ but the top down input $\Delta$ nudges them in opposite directions. The difference and weighted mean of these internal states are then propagated downstream and upstream respectively. \textbf{(b)} In a pyramidal neuron bottom-up signal arrive at the basal dendrites and top-down signal arrive at the apical dendrites. Concerns regarding DP and biological plausibility are discussed in section~\ref{sec:bioplausibility}.}
        \label{fig:dyadic_and_pyramidal_neuron}
\end{figure}

Dual propagation~\cite{hoier2023conf} (DP) is another algorithm in the NGRAD family, in which each neuron has two intrinsic states corresponding to positively and negatively nudged compartments as illustrated in Fig.~\ref{fig:dyadic_and_pyramidal_neuron}.
The algorithm is based on a specific network potential and uses the respective stationarity conditions to obtain an easy-to-implement inference method:
each neuron maintains two internal states representing the neural activity as (weighted) mean and the error signal as difference, respectively (which makes it an instance of an NGRAD method).
The mean state is sent ``upstream'' to the next layer while the difference is passed downstream to the preceding layer, where it nudges the receiving neural state.
The method essentially interleaves or ``braids'' the two inference phases and makes it possible to infer both states simultaneously. When the update sequence is chosen appropriately, as few as two updates per neuron are sufficient, making the algorithm comparable to back-propagation in terms of runtime and 10-100X faster than CHL and EP.
In practice, DP is applied to feed forward networks and EP typically to Hopfield models, which has implications for how inference is carried out (fixed-point method vs energy minimization) and on the hardware the algorithms are suitable for.
Another difference to EP and CHL is, that dual propagation infers both sets of states simultaneously (and not in sequential phases).
The underlying network potential used in DP is parametrized by a coefficient $\alpha\!\in\![0,1]$ determining the weighted mean.
The stability of the resulting inference method hinges on choosing $\alpha=1/2$.

Casting deep learning as an optimization task over explicit activations and weights is the focus of a diverse set of back-propagation alternatives sometimes collectively referred to as \emph{lifted neural networks}~\citep{carreira2014distributed,askari2018lifted,gu2020fenchellifted,li2020,choromanska2019beyond,zach2019,hoier2020}.
Predictive coding networks (e.g.~\citep{whittington2017approximation,salvatori2023brain}) can be also understood as instances of lifted neural networks.
Although members of this group have different origins, they are algorithmically closely related to CHL and EP~\citep{zach2021}, but vary e.g.\ in their suitability for digital hardware.

\paragraph{Weak and strong feedback}
While a number of CHL-inspired learning methods for neural networks are shown to be equivalent to back-propagation when the feedback parameter $\beta$ approaches zero (i.e.\ infinitesimal nudging takes place, as discussed in e.g.~\citep{xie2003,scellier2017,zach2019,zach2021}), practical implementations use a finite but small value for $\beta$, whose magnitude is further limited---either explicitly or implicitly. CHL implicitly introducing weak feedback via its layer discounting in order to approximate a feed-forward neural network, and both CHL and EP rely on weak feedback to stay in the same energy basin for the free and the clamped solutions. The iterative solver suggested for the LPOM model~\citep{li2020} also depends on sufficiently weak feedback to ensure convergence of the fixed-point inference scheme. In contrast to these restrictions, the feedback parameter $\beta$ in dual propagation is weakly constrained and its main effect is to influence the resulting finite difference approximation for activation function derivatives.

\section{Background}

A feed-forward network is a composition of $L$ layer computations
$s_{k-1} \mapsto f_k(W_{k-1}s_{k-1})$, where $f_k$ is the activation function
of layer $k$ and $W_{k-1}$ is the associated weight matrix. $s_0$ is the input
provided to the network and $s_k:= f_k(W_{k-1}s_{k-1})$ is the vector of
neural states in layer $k$. The set of weight matrices
$\theta=\{W_k\}_{k=1}^{L-1}$ form the trainable parameters, which are
optimized during the training stage to minimize a target loss $\ell$. Training
is usually---at least on digitial hardware---conducted via back-propagation of
errors, which is a specific instance of the chain rule.

In~\cite{hoier2023conf} a localized, derivative-free approach for supervised
learning in feed-forward neural networks is proposed, which is based on the following
network potential $\cL^{DP}_\alpha$ (Eq.~\ref{eq:L_DP}) for a parameter
$\alpha\in[0,1]$,
\begin{align}
  \label{eq:L_DP}
  \begin{split}
    \cL^{DP}_{\alpha}(\theta) = \min_{s^+} \max_{s^-} \alpha\ell(s_L^+) + \bar{\alpha}\ell(s_L^-)
    {} + \tfrac{1}{\beta}\sum\nolimits_{k=1}^L \Paren{ E_k(s_k^+, \Bar{s}_{k-1}) - E_k(s_k^-, \Bar{s}_{k-1}) }.
  \end{split}
\end{align}
The terms $E_k$ in $\cL^{DP}_\alpha$ are specifically chosen as
$E_k(s_k,s_{k-1}) = G_k(s_k) - s_k\tr W_{k-1} s_{k-1}$, where $G_k$ is
typically a strongly convex function for the ``resting'' energy (and relates to the activation function via $\nabla G_k=f_k^{-1}$ and $\nabla G_k^*=f_k$). This choice
leads---via the corresponding stationarity conditions---to particularly simple
fixed-point updates for $s_k^\pm$ for the inference of neural states,
\begin{align}
  \label{eq:old_DP_updates}
  \begin{split}
    s_k^+ &\gets f_k\!\Paren{ W_{k-1} \bar s_{k-1} + \alpha W_k\tr (s_{k+1}^+- s_{k+1}^-) } \\
    s_k^- &\gets f_k\!\Paren{ W_{k-1} \bar s_{k-1} - \bar\alpha W_k\tr (s_{k+1}^+- s_{k+1}^-) },
  \end{split}
\end{align}
where $\bar\alpha := 1-\alpha$ and
$\bar s_k := \alpha s_k^+ + \bar\alpha s_k^-$.
The state of the output units is determined by solving
\begin{align}
\begin{split}
    s_L^+ &\gets \arg\min_{s_L} \alpha\ell(s_L) + E_L(s_L, \bar s_{L-1}) \\
    s_L^- &\gets \arg\min_{s_L} -\bar\alpha\ell(s_L) + E_L(s_L, \bar s_{L-1}).
\end{split}
\label{eq:zL_updates_DP}
\end{align}
Gradient-based learning of the
weight matrices $\theta = (W_k)_{k=0}^{L-1}$ is based on the quantity
$\partial_{W_k} \cL^{DP}_\alpha \propto (s_{k+1}^+-s_{k+1}^-) \bar s_k\tr$ once the neural states $s_k^\pm$ are inferred.

The choice of $\alpha=1$ is equivalent to the LPOM formulation~\cite{li2020},
that builds on the following network potential,
\begin{align}
  \begin{split}
    \cL^{LPOM}(\theta) = \min_{s^+} \ell(s_L^+)
    + \tfrac{1}{\beta} \sum_{k=1}^L \Paren{ E_k(s_k^+, s_{k-1}^+) - E_k(f_k(W_{k-1}s_{k-1}^+), s_{k-1}^+) }.
  \end{split}
  \label{eq:LPOM}
\end{align}
Eq.~\ref{eq:LPOM} can be obtained from Eq.~\ref{eq:L_DP} by closed form
maximization w.r.t.\ all $s_k^-$, as $s_k^-$ in this setting solely depends on
$s_{k-1}^+$. Analogously, the choice $\alpha=0$ can be interpreted as
``inverted'' LPOM model (after minimizing w.r.t.\ all $s_k^+$
variables). Unfortunately, the straightforward update rules in
Eq.~\ref{eq:old_DP_updates} lead to a stable method only when
$\alpha=1/2$. In Section~\ref{sec:dampened_fixed_point} we discuss the necessary adaptation of Eq.~\ref{eq:old_DP_updates}
(which are also used in our experiments) when $\alpha\in\{0,1\}$.
In this work we (i) provide explanations why the choice $\alpha=1/2$ is
special (via connecting it to the adjoint state method in Section~\ref{sec:lagrangian} and via a more explicit analysis in Section~\ref{sec:fixed_point_analysis}) and (ii) demonstrate why selecting $\alpha\ne 1/2$ can be beneficial.

\section{A Relaxation Perspective on Dual Propagation}
\label{sec:relaxation}

In this section we show how $\cL^{DP}_\alpha$ can be obtained exactly by
repeated application of certain relaxations of the optimal value
reformulation~\cite{outrata1988note,dempe2013bilevel} (summarized in Section~\ref{sec:ROVR}). Obtaining
$\cL^{DP}_\alpha$ requires the application of two (to our knowledge) novel relaxations for
bilevel programs as described in Sections~\ref{sec:SPROVR} and~\ref{sec:AROVR}.

\subsection{The Optimal Value Reformulation and its Relaxation}
\label{sec:ROVR}

First we describe a basic (and known) reformulation for a class of bilevel
minmization problems. We consider a general bilevel program of the form
\begin{align}
  \min_\theta \ell(s^*) \qquad \text{s.t. } s^* = \arg\min_s E(s;\theta)
  \label{eq:bilevel}
\end{align}
In the context of this work we assume that the minimizer of
$E(s;\theta)$ is unique for all $\theta$, which corresponds to layer
computation in deep neural networks to be proper (single-valued) functions.
In a first step the constraint in Eq.~\ref{eq:bilevel} can restated as
\begin{align}
  \min_{\theta,s} \ell(s) \qquad \text{s.t. } E(s;\theta) \le \min_{s'} E(s';\theta),
\end{align}
which is the \emph{optimal value
  reformulation}~\cite{outrata1988note,dempe2013bilevel}. As
$E(s;\theta) \ge \min_{s'} E(s';\theta)$ by construction, the inequality
constraint is always (possibly only weakly) active.  By fixing the corresponding
non-negative Lagrange multiplier to $1/\beta$ for a $\beta>0$ we obtain a
\emph{relaxed optimal value reformulation} (``ROVR''),
\begin{align}
  \begin{split}
     \min_{\theta,s} \ell(s) + \tfrac{1}{\beta} \Paren{ E(s;\theta) - \min_{s'} E(s';\theta) } 
    = \min_{\theta,s} \max_{s'} \ell(s) + \tfrac{1}{\beta} \big( E(s;\theta) - E(s';\theta) \big).
  \end{split}
  \label{eq:ROVR}
\end{align}
Repeated application of this conversion from lower-level minimization task to
higher-level penalizers on deeply nested problems yields a large variety of
known learning methods~\cite{zach2021}. When $\beta\to 0$, a number of works
have shown the connection between Eq.~\ref{eq:ROVR}, back-propagation and
implicit differentiation~\cite{xie2003,scellier2017,gould2019deep,zach2021}.

\subsection{A Saddlepoint Relaxed Optimal Value Reformulation}
\label{sec:SPROVR}

A different relaxation for a bilevel program can be obtained by rewriting
Eq.~\ref{eq:bilevel} as follows,
\begin{align}
  \label{eq:minmax_bilevel}
  \begin{split}
    \min_{s^+} \max_{s^-} \ell(\alpha s^+ \!+\! \bar\alpha s^-) \;\; \text{s.t. } s^+ &= \arg\min_s E(s;\theta) \\
    s^- &= \arg\min_s E(s;\theta),
  \end{split}
\end{align}
where $\alpha\in [0,1]$ and $\bar\alpha=1-\alpha$ as before. The reformulation
above replaces the solution $s^*$ of the inner problem with an (apparently
superfluous) convex combination of two related solutions $s^+$ and
$s^-$. Under the uniqueness assumption on the minimizer of
$E(\cdot;\theta)$, this program is equivalent to Eq.~\ref{eq:bilevel} (as
$s^*=s^+=s^-$ by construction). The main effect of introducing $s^\pm$ in
Eq.~\ref{eq:minmax_bilevel} is, that it allows two independent applications of
the ROVR. Applying the standard ROVR on $s^-$ (a maximization task) yields the
first relaxation,
\begin{align}
  \min_{s^+} \max_{s^-}{} &\ell(\alpha s^+\!+\!\bar\alpha s^-) - \tfrac{1}{\beta} (E(s^-;\theta) - \min\nolimits_s E(s;\theta) )
                         \qquad\text{s.t. } s^+ = \arg\min\nolimits_s E(s;\theta),
\end{align}
and a second application of the ROVR on $s^+$ (with the same multiplier
$1/\beta$) results in
\begin{align}
  \begin{split}
    J_{\alpha,\beta}^{\text{SPROVR}}(\theta) := \min_{s^+} &\max_{s^-} \ell(\alpha s^+\!+\!\bar\alpha s^-)
    + \tfrac{1}{\beta} (E(s^+;\theta) - E(s^-;\theta))
  \end{split}
  \label{eq:SPROVR}
\end{align}
as the $\min\nolimits_s E(s;\theta)$ terms cancel each other. We can reverse
the order of min and max in Eq.~\ref{eq:minmax_bilevel}, which also induces a
reversal of min and max in Eq.~\ref{eq:SPROVR}. Strong duality may not be
satisfied in general, but the stationarity conditions w.r.t.\ $s^+$ and
$s^-$ are unaffected by permuting min and max.
We call the algebraic conversion from Eq.~\ref{eq:bilevel} to Eq.~\ref{eq:SPROVR} the \emph{saddle-point ROVR or SPROVR}.

In order to model deep feed-forward networks we extend the bilevel program to
a multi-level program with with $L$ nesting levels,
\begin{align}
  \begin{split}
    \min_{\theta} \ell(s_L^*) \qquad \text{s.t. } s_k^* &= \arg\min\nolimits_{s_k} E_k(s_k,s_{k-1}^*;\theta) \\
    s_1^* &= \arg\min\nolimits_{s_1} E_1(s_1;\theta),
  \end{split}
  \label{eq:deeply}
\end{align}
where each $E_k(\cdot,s_{k-1};\theta)$ is assumed to have a unique minimizer
for all $s_{k-1}$ and $\theta$. By applying SPROVR repeatedly from
the outermost level to inner ones, we arrive at the following relaxed reformulation,
\begin{align}
  \begin{split}
    \min_\theta \min_{\{s_k^+\}} \max_{\{s_k^-\}} &\,\ell(\bar s_L) + \tfrac{1}{\beta} \Paren{ E_1(s_1^+;\theta) - E_1(s_1^-;\theta) } \\
    &+ \tfrac{1}{\beta} \sum\nolimits_{k=2}^L \Paren{ E_k(s_k^+;\bar s_{k-1};\theta) - E_k(s_k^-;\bar s_{k-1};\theta) }.
  \end{split}
\end{align}
We recall that the short-hand notation $\bar s_k$ stands for
$\bar s_k=\alpha s_k^+ + \bar\alpha s_k^-$, but the choice of the coefficient
$\alpha$ can in principle be layer-dependent (by replacing $\alpha$ with
$\alpha_k$). Similarly, the value of multiplier $1/\beta$ may depend on the
layer $k$. Traversing the SPROVR from the inside towards the outermost level
leads to a different (and less relevant) relaxation.

\subsection{An Adversarial Relaxed Optimal Value Reformulation}
\label{sec:AROVR}

In this section we present a different instance of an optimal value
reformulation, that yields centered approaches to contrastive Hebbian learning
and is further closely linked to adversarial training. We start with the
bilevel program,
\begin{align}
    J_{\alpha,\beta}^{\text{adv}}(\theta) = \ell(s^*) \quad \text{s.t. } s^* = \arg\min_s E(s;\theta) - \bar\alpha\beta \ell(s),
  \label{eq:adv_bilevel}
\end{align}
where $\beta>0$ is a step size parameter and $\alpha\in[0,1]$ further modulates the step length. The solution $s^*$ is a minimizer of an adversarially perturbed inner problem (and assumed to exist), and therefore Eq.~\ref{eq:adv_bilevel} is generally not equivalent to Eq.~\ref{eq:bilevel}.
A simple illustrative example is given by setting $E(s;\theta) = \norm{s-\theta}^2/2$ and $\ell(s)=g\tr s$, then $\arg\min_s E(s;\theta)=\theta$ but $s^*$ in Eq.~\ref{eq:adv_bilevel} is given by $s^* = \theta + \bar\alpha\beta g$, i.e. $s^*$ takes a gradient ascent step from $\theta=\arg\min_s E(s;\theta)$ to increase the (linear) target loss $\ell$ (somewhat in analogy with the fast gradient method). On the other hand, the outer problem aims to reduce the main loss $\ell$ for this perturbed sample $s^*$.

Using the optimal value reformulation (cf.\ Section~\ref{sec:ROVR}) this is equivalent to an inequality constrained program,
\begin{align}
  \begin{split}
     \min_{s^*} \ell(s^*)
     &\;\text{s.t. } E(s^*;\theta) - \bar\alpha\beta\ell(s^*) \le \min_s E(s;\theta) - \bar\alpha\beta\ell(s).
  \end{split}
\end{align}
Fixing the (non-negative) multiplier of the Lagrangian relaxation to $1/\beta$ yields (after algebraic manipulations)
\begin{align}
\begin{split}
    J_{\alpha,\beta}^{\text{AROVR}}(\theta) := \min_{s^+} &\max_{s^-}{} \alpha\ell(s^+) + \bar\alpha\ell(s^-)
    + \tfrac{1}{\beta} \Paren{ E(s^+;\theta) - E(s^-;\theta) },
\end{split}
\end{align}
which e.g.\ generalizes the centered variant of EP. Increasing the effective step length $\bar\alpha\beta$ in $J_{\alpha,\beta}^{\text{adv}}$ by increasing $\bar\alpha$ allows larger adversarial perturbations and therefore makes the learning task harder. This property carries over to $J_{\alpha,\beta}^{\text{AROVR}}$ as shown in the following proposition (proven in the appendix).
\begin{proposition}
  \label{thm:adv_CHL}
  Let $0 < \beta \le \beta'$ and $\alpha, \alpha'\in[0,1]$ with
  $\alpha\le\alpha'$. Then the following holds: (a)
  $J_{\alpha',\beta}^{\text{AROVR}}(\theta) \le
  J_{\alpha,\beta}^{\text{AROVR}}(\theta)$ and (b)
  $\beta J_{\alpha,\beta}^{\text{AROVR}}(\theta) \le \beta'
  J_{\alpha,\beta'}^{\text{AROVR}}(\theta)$.
\end{proposition}
Original EP ($\alpha=1$) and centered EP~\cite{laborieux2021scaling,scellier2023energybased} ($\alpha=1/2$) are two
particular instances of $J_{\alpha,\beta}^{\text{AROVR}}$. Hence,
Prop.~\ref{thm:adv_CHL} raises the question of how much the better gradient
estimate and how much the stricter loss $J_{1/2,\beta}^{\text{AROVR}}$
contribute to the advantages of centered~EP over standard EP.

Without further assumption we cannot expect claim (b) to be strengthened to
$J_{\alpha,\beta}^{\text{AROVR}}(\theta) \le
J_{\alpha,\beta'}^{\text{AROVR}}(\theta)$ as a larger adversarial step is offset
by a smaller multiplier $1/\beta'$. If the total step size
$\bar\alpha \beta$ remains constant, then a stronger relation can nevertheless be shown:
\begin{proposition}
  \label{thm:adv_CHL_2}
  Let $0 < \beta' < \beta$ and $\alpha'$ such that
  $\bar\alpha\beta=\bar\alpha'\beta'$. Then
  $J_{\alpha,\beta}^{\text{AROVR}}(\theta) \le
  J_{\alpha',\beta'}^{\text{AROVR}}(\theta)$.
\end{proposition}
The two relaxations, $J_{\alpha,\beta}^{\text{AROVR}}$ and
$J_{\alpha,\beta}^{\text{SPROVR}}$ (Eq.~\ref{eq:SPROVR}) look very similar,
the only difference being the location of the averaging step (averaging the
argument or the function value of the target loss $\ell$). Generally we expect
$J_{\alpha,\beta}^{\text{AROVR}}$ to upper bound
$J_{\alpha,\beta}^{\text{SPROVR}}$, which is always the case when
$\ell$ is convex. The huge advantage of SPROVR is that it leads to tractable
objectives when applied to deeply nested problems (linear in the problem depth
$L$). Applying AROVR in a similar manner on Eq.~\ref{eq:deeply} yields an exponential
number of terms in the resulting objective unless $\alpha=1$ or
$\alpha=0$ (where it actually coincides with SPROVR). We now can easily
identify $\cL_\alpha^{DP}$ (Eq.~\ref{eq:L_DP}) as result of applying both
SPROVR and AROVR on a deeply nested program:
\begin{corollary}
  The DP objective $\cL_\alpha^{DP}$ (Eq.~\ref{eq:L_DP}) can be obtained by
  first applying AROVR on the outermost level of Eq.~\ref{eq:deeply}, followed
  by subsequent applications of SPROVR for the remaining levels.
\end{corollary}

In Section~\ref{sec:exp_strong_feedback} we verify whether different choices for $\alpha$ have any impact on the robustness of a trained model. Since any differences are expected to vanish for small choices of $\beta$, the large $\beta$ setting (we choose $\beta=1/2$) is of main interest.
Note that unlike standard adversarial training (which perturbs solely the input to the network), the lifted network potential $\cL_\alpha^{DP}$ is based on perturbations of the network's hidden activations. Adversarial perturbations of the inputs can be achieved by adding an auxiliary constraint $s_0=x$, e.g.\ by extending ${\cal L}^{DP}_\alpha$ with e.g.\ $E_0(s_0,x)=\norm{s_0-x}^2$.

In our experiments, a significant reduction of the Lipschitz estimate is in fact observed when decreasing $\alpha$ (and therefore increasing the ``adversarial'' step size $\bar\alpha\beta$).
We do of course not suggest to replace advanced adversarial training with the loss $J^{\text{AROVR}}_{0,\beta}$, but the main goal of this section is to demonstrate that a non-infinitesimal choice of $\beta$ has a theoretical and practical impact and therefore goes beyond a pure approximation of back-propagation (respectively implicit differentiation). One may also speculate that the hypothesized NGRAD-based learning in biological systems gains some robustness by leveraging a similar mechanism.

\section{A Lagrangian Perspective on Dual Propagation}
\label{sec:lagrangian}

In this section we derive a variant of dual propagation, which we refer to as DP$\tr$, which turns out to be robust to asymmetric nudging (i.e.\ choosing $\alpha \in[0,1]$ not necessarily equal to $1/2$).
Our starting point is a modified version of LeCun's classic Lagrangian-based 
derivation of backpropagation~\citep{lecun1988}.
We assume (i) that the activation functions $f_k$ are all invertible (which can be relaxed), and (ii) that $f_k$ has a symmetric derivative (i.e.\ $f_k'(x) = f_k'(x)\tr$). The second assumption clearly holds e.g.\ for element-wise activation functions.
Our initial Lagrangian relaxation can now be stated as follows,
\begin{align}
  \label{eq:Lagrangian}
  \begin{split}
    \cL(\theta) &= \min_s \max_\delta \ell(s_L)
    + \sum\nolimits_{k=1}^L \delta_k\tr \!\Paren{ f_k^{-1}(s_k) - W_{k-1} s_{k-1} }.
  \end{split}
\end{align}
Note that the multiplier $\delta_k$ corresponds to the constraint $f_k^{-1}(s_k) = W_{k-1} s_{k-1}$ (in contrast to the constraint $s_k=f_k(W_{k-1}s_{k-1})$ employed in~\citep{lecun1988}). The main step is to reparamtrize $s_k$ and $\delta_k$ in terms of $s_k^+$ and $s_k^-$,
\begin{align}
  \label{eq:reparametrization}
   s_k = \alpha s_k^+ + \bar\alpha s_k^- & & \delta_k = s_k^+ - s_k^-
\end{align}
for a parameter $\alpha\in[0,1]$ (and $\bar\alpha := 1-\alpha$). In the following we use the short-hand notations $\bar s_k := \alpha s_k^+ + \bar\alpha s_k^-$ and $a_k:=W_{k-1}\bar s_{k-1}$.

\begin{proposition}
  \label{prop:dual_prop}
  Assume that the activation functions $f_k$, $k=1,\dotsc,L-1$, are invertible, have symmetric derivatives and behave locally linear. Then the Lagrangian corresponding to the reparametrization in Eq.~\ref{eq:reparametrization} is given by
  \begin{align}
    \label{eq:new_Lagrangian}
    \begin{split}
    \cL^{\text{DP}\tr}_\alpha&(\theta) = \min_{s^+} \max_{s^-} \ell(\bar s_L)
    + \sum_{k=1}^L (s_k^+\!-\!s_k^-)\tr\!\Paren{ f_k^{-1}(\bar s_k) \!-\! W_{k-1} \bar s_{k-1} },
    \end{split}
  \end{align}
  and the optimal $s_k^\pm$ in~\eqref{eq:new_Lagrangian} satisfy
  \begin{align}
    \label{eq:new_zk_updates}
    \begin{split}
      s_k^+ &\gets f_k\!\Paren{ W_{k-1} \bar s_{k-1} + \bar\alpha W_k\tr (s_{k+1}^+-s_{k+1}^-) } \\
      s_k^- &\gets f_k\!\Paren{ W_{k-1} \bar s_{k-1} - \alpha W_k\tr (s_{k+1}^+-s_{k+1}^-) }
    \end{split}
  \end{align}
  for internal layers $k=1,\dotsc,L-1$.
\end{proposition}

For the output layer activities $s_L^\pm$ we absorb $f_L$ into the target loss (if necessary) and therefore solve
\begin{align}
  \min_{s_L^+} \max_{s_L^-} \ell(\bar s_L) + (s_L^+-s_L^-)\tr (\bar s_L - a_L)
\end{align}
with corresponding necessary optimality conditions
\begin{align}
  \begin{split}
    0 &= \alpha \ell'(\bar s_L) + \bar s_L - a_L + \alpha (s_L^+-s_L^-) \\
    0 &= \bar\alpha \ell'(\bar s_L) - \bar s_L + a_L + \bar\alpha (s_L^+-s_L^-).
  \end{split}
  \label{eq:zL_optimality}
\end{align}
Adding these equations reveals $\ell'(\bar s_L) + s_L^+ - s_L^- = 0$. Inserting this
in any of the equations in~\eqref{eq:zL_optimality} also implies $\bar s_L = a_L = W_{L-1}\bar s_{L-1}$.
Thus, by combining these relations and with $g:=\ell'(W_{L-1}\bar s_{L-1})$, the updates for $s_L^\pm$ are given by
\begin{align}
  \label{eq:zL_updates_DPplus}
  s_L^+ \gets W_{L-1}\bar s_{L-1} - \bar\alpha g & & s_L^- \gets W_{L-1}\bar s_{L-1} + \alpha g.
\end{align}
For the $\beta$-weighted least-squares loss, $\ell(s_L) = \frac{\beta}{2} \norm{s_L-y}^2$, the above updates reduce to $s_L^+ \gets a_L - \bar\alpha \beta (a_L-y)$ and $s_L^- \gets a_L + \alpha \beta (a_L-y)$.

The neural activities $s_k^\pm$ together encode the forward and the adjoint
state. Since the updates in Eq.~\ref{eq:new_zk_updates} are based on the
adjoint state method, we denote the resulting algorithm adjoint-DP or just DP$\tr$.

\paragraph{Relation to the original dual propagation}
The update equations in~\eqref{eq:new_zk_updates} coincide with the original dual propagation rules if $\alpha=1/2$~\citep{hoier2023conf}, although the underlying objectives $\cL^{\text{DP}\tr}_\alpha$~\eqref{eq:new_Lagrangian} and $\cL^{\text{DP}}_\alpha$ are fundamentally different.
If $\alpha\ne 1/2$, then $\alpha$ and $\bar\alpha$ switch places w.r.t.\ the error signals (but not in the definition of $\bar s$) compared to the original dual propagation method.\footnote{This partial exchange of roles of $\alpha$ and $\bar\alpha$ is somewhat analogous to the observation that e.g.\ forward differences in the primal domain become backward differences in the adjoint. }
The updates in~\eqref{eq:new_zk_updates} for $\alpha=0$ correspond to an algorithm called ``Fenchel back-propagation'' discussed in~\citep{zach2021}.

Both the objective of the original dual propagation~\eqref{eq:L_DP} and the objective of the improved variant~\eqref{eq:new_Lagrangian} can be expressed in a general contrastive form,
but the underlying potentials $E_k$ are somewhat different,
\begin{align}
  \text{DP: }
   E_k(s_k^+, \bar{s}_{k-1}) \!-\! E_k(s_k^-, \bar{s}_{k-1})
  &= G_k(s_k^+) - G_k(s_k^-) - (s_k^+ \!-\! s_k^-)\tr W_{k-1} \bar s_{k-1} \nonumber \\
  \text{DP$\tr$: }
  E_k(s_k^\pm, \bar{s}_{k-1}) \!-\! E_k(s_k^\mp, \bar{s}_{k-1})
  &=(s_k^+ \!-\! s_k^-)\tr\nabla G_k(\bar{s}_k) - (s_k^+ \!-\! s_k^-)^\top W_{k-1} \bar s_{k-1}. \nonumber
\end{align}
Here $G_k$ is a convex mapping ensuring that $\argmin_{s_k}E_k(s_k,s_{k-1})$ provides the desired activation function $f_k$. The relation between $f_k$ and $G_k$ is given by $f_k = \nabla G_k^*$ and $f_k^{-1} = \nabla G_k$ (where $G_k^*$ is the convex conjugate of $G_k$).
Activations function induced by $G_k$ have automatically symmetric Jacobians as $f_k' = \nabla^2 G_k^*$ under mild assumptions.
The main difference between these two flavors of DP can be restated as follows,
\begin{align}
\begin{split}
    \Delta E_k &= G_k(s_k^+) - G_k(s_k^-) - (s_k^+ \!-\! s_k^-)\tr\nabla G_k(\bar{s}_k) 
        = D_{G_k}(s_k^+\|\bar s_k) - D_{G_k}(s_k^- \| \bar s_k),
\end{split}
\end{align}
where $D_{G_k}$ is the Bregman divergence induced by $G_k$. The difference of Bregman divergences can have any sign, and $\Delta E_k=0$ if $D_{G_k}(s_k^+\|\bar s_k) = D_{G_k}(s_k^- \| \bar s_k)$. Hence, for the choice $\alpha\in\{0,1\}$ (i.e.\ $\bar s_k = s_k^+$ or $\bar s_k=s_k^-$), $\Delta E_k$ can only be zero when $s_k^+ = s_k^-$. Fixing $\alpha=1/2$ and $G_k$ to any convex quadratic function implies $\Delta E_k=0$ (as $D_{G_k}$ becomes the squared Mahalanobis distance in this setting).

\paragraph{Non-invertible activation function}
If $f_k$ is not invertible at the linearization point (such as the softmax function), then $D_k$ is singular and the constraint $f_k^{-1}(s_k) = W_{k-1}s_{k-1}$ is converted into a constraint that $D_k^+ s_k$ is restricted to a linear subspace,
\begin{align}
    D_k^+ (s_k-s_k^0) = W_{k-1} s_{k-1} + N_k v_k,
\end{align}
where $N_k$ spans the null space of $D_k$ and $v_k$ is an arbitrary vector. Going through the similar steps as above leads to the same updates for $s_k^\pm$.

\paragraph{Starting from Lecun's Lagrangian}
If we start from the more natural Lagrangian
\begin{align}
  \begin{split}
    \cL(\theta) &= \min_s \max_\delta \ell(s_L) 
    + \sum\nolimits_{k=1}^L \delta_k\tr \!\Paren{ s_k - f_k(W_{k-1} s_{k-1}) },
  \end{split}
\end{align}
instead from~\eqref{eq:Lagrangian}, then 
the back-propagated signal $s_{k+1}^--s_{k+1}^+$ cannot be easily moved inside the activation function $f_k$. The update equations for $s_k^\pm$ are now of the less-convenient form
\begin{align}
\begin{split}
    s_k^+ &\gets f_k(W_{k-1}\bar s_{k-1}) + \bar{\alpha} W_k\tr f_{k+1}'(\bar s_k)(s_{k+1}^-\!-\!s_{k+1}^+)\\
    s_k^- &\gets f_k(W_{k-1}\bar s_{k-1}) - \alpha W_k\tr f_{k+1}'(\bar s_k)(s_{k+1}^-\!-\!s_{k+1}^+) , 
\end{split} \nonumber
\end{align}
which still require derivatives of the (next-layer) activation functions, which makes it less appealing to serve as an NGRAD method.

\section{Numerical Validation}

Our numerical experiments aim to verify two claims: first, in Section~\ref{sec:exp_strong_feedback} we demonstrate that choosing different values for $\alpha$ makes a significant difference in the learned weight matrices. Second, Section~\ref{sec:vgg16} validates that the proposed DP$\tr$ method is efficient enough to enable successful training beyond toy-sized DNNs for arbitrary choices of $\alpha$~\footnote{The code used in our experiments is available at: \url{github.com/Rasmuskh/dualprop_icml_2024}}.

\paragraph{Implementation}
The adjoint variant DP$\tr$ literally implements the updates stated in
Eq.~\ref{eq:new_zk_updates}. In order to enable DP for values
$\alpha\ne 1/2$, we leverage stabilized fixed-point updates
(Section~\ref{sec:dampened_fixed_point}), 
Not using stabilized updates yields failures to learn successfully as indicated in Table~\ref{tab:results} (with details given in Section~\ref{sec:DP_failure}).
Using stable fixed-point updates comes at a significant computational cost as more iterations are required to obtain sufficiently accurate neural states. Consequently we apply these only in the experiments related to the Lipschitz behavior described in Section~\ref{sec:exp_strong_feedback}.

\subsection{The impact of $\alpha$ on the Lipschitz continuity}
\label{sec:exp_strong_feedback}

\def\advsize{\tiny}
\def\advspacing{\,\,\,}

\begin{table*}[t]
\setlength{\tabcolsep}{2pt}
  \centering
  \footnotesize
  \caption{Test accuracy and Lipschitz properties of an MLP trained with DP and DP$\tr$ for different choice of $\alpha$ (and fixed $\beta=1/2$). $\hat L$ is the spectral norm of the product of weight matrices. Results are averaged over 3 random seeds. The additional results under DP and BP correspond to adding input layer robustness (DP) and FGSM-based adversarial training (BP), respectively. }
  \begin{tabular}{l|lrrr|rrr||r}
    \hline
    & \multicolumn{4}{c|}{DP} & \multicolumn{3}{c||}{DP$\tr$} & \multicolumn{1}{c}{BP} \\ 
    & $\alpha$ & \multicolumn{1}{c}{1} & \multicolumn{1}{c}{\textonehalf}
    & \multicolumn{1}{c|}{0} & \multicolumn{1}{c}{1} & \multicolumn{1}{c}{\textonehalf} & \multicolumn{1}{c||}{0} & \\ \hline
    MNIST   & Acc      & $95.85\cpm0.68$ & $98.09\cpm0.06$ & $98.46\cpm0.07$ & $98.06\cpm0.15$ & $97.9\cpm0.1~~$ & $98.21\cpm0.13$ & $98.05\cpm0.12$ \vspace{-0.15em} \\
            &          &                 &           & \advsize$98.43\cpm0.11\advspacing$ & & & & \advsize$98.43\cpm0.04\advspacing\,\,$ \\
            & $\hat L$ & $522\cpm288~$ & $67.2\cpm12.6$  & $29.9\cpm4.2~~$ &  $73.4\cpm2.7~~$  & $78.1\cpm9.9~~$ & $43.2\cpm19.4$  & $55.1\cpm14.4$ \vspace{-0.15em} \\
          &          &               &           & \advsize$15.87\cpm6.34\advspacing$ & & &  & \advsize$36.16\cpm4.23\advspacing\,\,$ \\
    \hline
    F-MNIST & Acc      & $83.36\cpm 0.74$ & $88.32\cpm0.05$ & $89.16\cpm0.43$  & $88.26\cpm0.46$ & $88.35\cpm0.37$ & $88.18\cpm0.27$ & $88.37\cpm0.56$ \vspace{-0.15em} \\
            &          &                  &         & \advsize$88.29\cpm0.79\advspacing$ & & & & \advsize$88.64\cpm0.25\advspacing\,\,$ \\
    {}      & $\hat L$ & $157.7\cpm57.8$  & $34.3\cpm4.0~~$ & $19.97\cpm3.9~~$ & $52.0\cpm23.6$  & $35.1\cpm7.0~~$ & $29.4\cpm2.7~~$ & $37.7\cpm10.2$ \vspace{-0.15em} \\
    {}      &          &                  &         & \advsize$8.05\cpm1.03\advspacing$ & & & & \advsize$12.68\cpm0.81\advspacing\,\,$ \\
    \hline
  \end{tabular}
  \label{tab:lipschitz}
\end{table*}

Section~\ref{sec:AROVR} suggests that the choice of $\alpha$ matters in particular in the strong feedback setting (i.e.\ $\beta$ is not close to zero). We ran experiments on MNIST and FashionMNIST using a 784-512($\!\times\!$ 2)-10 MLP with ReLU activation functions, and Table~\ref{tab:lipschitz} reports test accuracy and Lipschitz estimates (computed as spectral norm of the accumulated weight matrix $W_2W_1W_0$ as ReLU is 1-Lipschitz) after 20 epochs of ADAM-based training (with learning rate $0.001$ and default parameters otherwise). No additional loss terms such as weight decay were used.
As indicated by Prop.~\ref{thm:adv_CHL}, a lower value of $\alpha$ yields a significantly smaller Lipschitz constant (as well as improved test accuracy) for the DP method. The DP$\tr$ approach leads to a similar but far less pronounced behavior regarding the Lipschitz estimate. The small differences between the Lipschitz estimates for the different choices of $\alpha$ have no consistent impact on the test accuracy.

As pointed out in Section~\ref{sec:AROVR}, the DP objective $\cL^{DP}_{\alpha=0}$ is robust only w.r.t.\ perturbations in the internal layers. Consequently, the very first layer is not necessarily adversarially robust, and additional robustness can be achieved by extending $\cL^{DP}_0$ with a term for the input layer, introducing another term $E_0(s,x) = \norm{s-x}^2/2$. Table~\ref{tab:lipschitz} also includes the resulting accuracy and Lipschitz estimates for this extended variant of $\cL_0^{DP}$ as well as the results for simple adversarial training (using FGSM~\cite{goodfellow2014adversarial_examples} as approximate inner maximization; the step size $\eps=0.01$ chosen to roughly match the test accuracy of DP). Adding input layer robustness significantly reduces the Lipschitz estimate of the DP-trained network (at a small reduction of test accuracy).
In the setting $\alpha=1/2$, DP and DP$\tr$ use identical update rules for the hidden units, but treat the output units differently (Eq.~\ref{eq:zL_updates_DP} vs.\ Eq.~\ref{eq:zL_updates_DPplus}), hence the two methods behave very similar but not identical in practice.
Finally, we remark that the difference between all tested methods largely vanishes in the weak feedback setting as they all mimic back-propagation when $\beta\to 0^+$.

\begin{table}[tb]
    \centering
    \caption{Mean test accuracy in percent for the original and the improved dual propagation methods using $\alpha\in\{0,1\}$. (X) indicates that the particular experiment did not converge. Results are averaged over five random seeds.}
    \label{tab:results}
    \begin{tabular}{l|ll|ll}
        \hline
                   & \multicolumn{2}{c|}{$\alpha=0$} & \multicolumn{2}{c}{$\alpha=1$} \\ \hline
        Iters & DP$\tr$    & DP   & DP$\tr$         & DP        \\ \hline
        1          & $98.3\cpm0.1$ & $98.50\cpm0.1$& $98.4 \cpm0.1$     & $97.9\cpm0.1$     \\ \hline
        30         & $98.4\cpm0.0$ & X             & $98.5\cpm0.1$      & X                \\ \hline
    \end{tabular}
\end{table}

\begin{table*}[tb]
\caption{Table of top-1 and top-5 accuracies in percent for VGG16 networks trained with DP$\tr$ for different choices of $\beta$ and $\alpha$.
($\mathrm{\dagger}$) Indicates an experiment where a reasonable model was found before training eventually diverged. (X) indicates an experiment where no reasonable model was learned.}
\label{tab:cifar}
\scriptsize
\setlength{\tabcolsep}{2pt}
\centering
\begin{tabular}{ll|lll|lll|lll||l}
\hline
                          & $\beta$ & \multicolumn{3}{c|}{$1.0$}      & \multicolumn{3}{c|}{$0.1$} & \multicolumn{3}{c||}{$0.01$} & \multirow{2}{*}{BP}\\ 
                          & $\alpha$ & $0.0$    & $0.5$    & $1.0$       & $0.0$    & $0.5$    & $1.0$   & $0.0$    & $0.5$    & $1.0$    \\ \hline
CIFAR10                   & Top-1 & $ 91.34 \cpm 0.16 $ & X & X & $ 92.41 \cpm 0.07 $       & $\mathrm{\dagger} 89.3 \cpm 0.17 $       & X      & $ 92.26 \cpm 0.22 $       & $ 92.19 \cpm 0.32 $       & $ 92.19 \cpm 0.08 $ & $ 92.36 \cpm 0.16 $       \\ \hline
\multirow{2}{*}{CIFAR100} & Top-1 & $ ~69.3 \pm 0.08 $ & X & X & $ 70.05 \cpm 0.23 $       & $ 69.25 \cpm 0.28 $       & X      & $ 69.33 \cpm 0.24 $       & $ 69.38 \cpm 0.10 $       & $ 68.83 \cpm 0.06 $ & $ ~~69.1 \cpm 0.1 $      \\
                          & Top-5 & $ ~89.6 \pm 0.11 $ & X & X & $ 89.42 \cpm 0.16 $       & $ 88.54 \cpm 0.18 $      & X      &  $ 88.46 \cpm 0.17 $       & $ 88.23 \cpm 0.18 $       & $ 88.28 \cpm 0.13 $ & $ 88.24 \cpm 0.22 $      \\ \hline
\end{tabular}
\end{table*}

\begin{figure*}[t]
     \centering
     \begin{subfigure}[b]{0.495\textwidth}
         \centering
         \includegraphics[width=\textwidth]{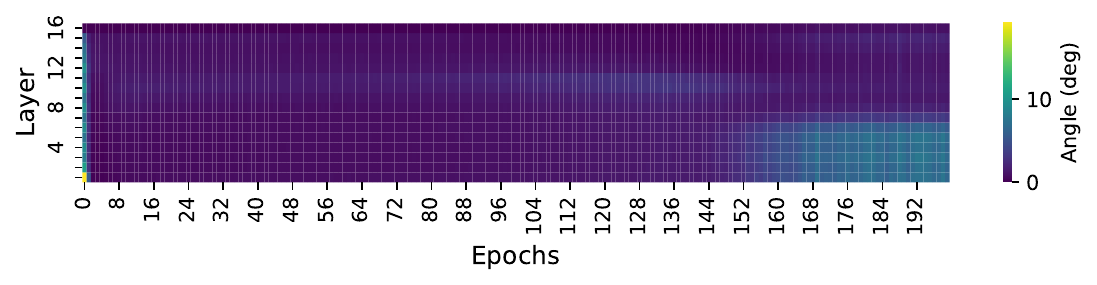}
         \caption{$\alpha=0.0$, $\beta=0.01$}\label{fig:grad_angle_epochs_beta001}
     \end{subfigure}
    \hfill
     \begin{subfigure}[b]{0.495\textwidth}
         \centering
         \includegraphics[width=\textwidth]{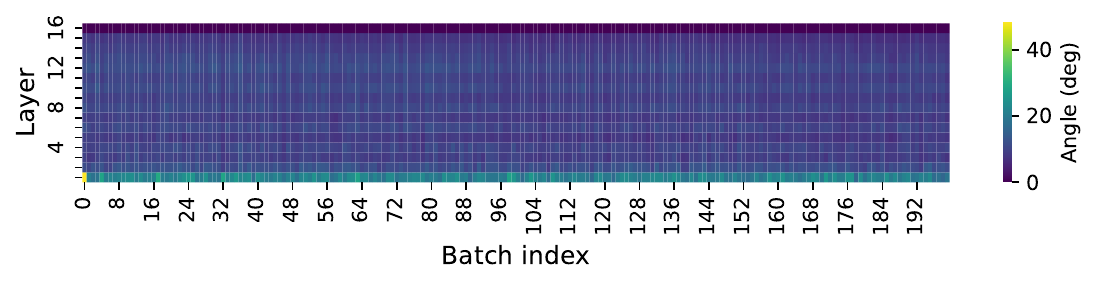}
         \caption{$\alpha=0.0$, $\beta=0.01$}\label{fig:grad_angle_batches_beta001}
     \end{subfigure}
     \hfill
     \begin{subfigure}[b]{0.495\textwidth}
         \centering
         \includegraphics[width=\textwidth]{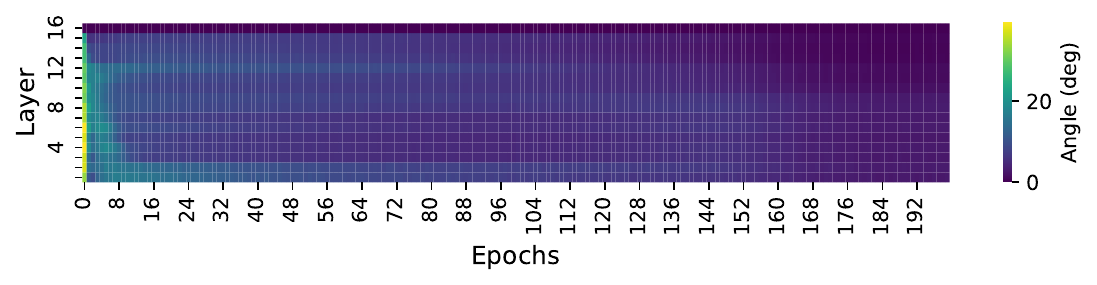}
         \caption{$\alpha=0.0$, $\beta=1.0$}\label{fig:grad_angle_epochs_beta1}
     \end{subfigure}
    \hfill
     \begin{subfigure}[b]{0.495\textwidth}
         \centering
         \includegraphics[width=\textwidth]{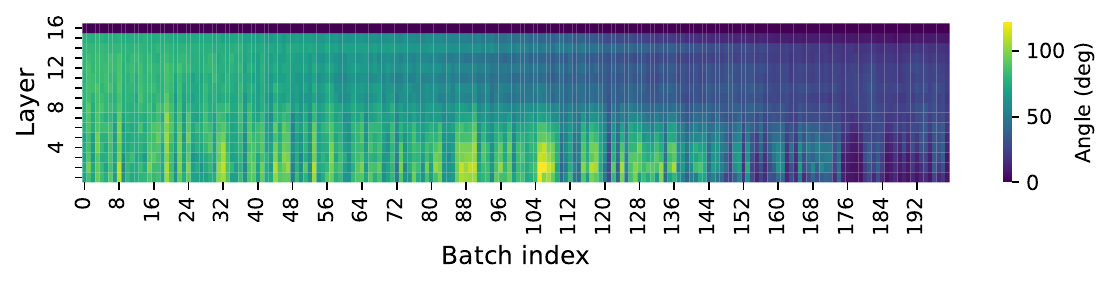}
         \caption{$\alpha=0.0$, $\beta=1.0$}\label{fig:grad_angle_batches_beta1}
     \end{subfigure}
     \hfill
        \caption{CIFAR100: Angle of gradient estimates relative to BP gradients, for DP$\tr$ with $\alpha=0$ and $\beta=0.01$ and $\beta=1.0$. Angles are plotted across layers and epochs (\textbf{left column}). The \textbf{right column} zooms in on the first 200 minibatches (i.e.\ a fifth of an epoch).}
        \label{fig:grad_angle}
\end{figure*}

\subsection{VGG16 experiments}
\label{sec:vgg16}
We also train a 16 layer VGG network using DP$\tr$ with a crossentropy classification loss on the CIFAR10 and CIFAR100 datasets. 
In contrast to~\cite{hoier2023conf} we do not employ the efficient, backpropagation-like forward-backward traversal of the layers. 
Instead we apply 17 forward-only passes through the layers, which---due to the dyadic neurons---allows errors to gradually propagate to all layers. We employ standard data augmentation (random crops and horizontal flips) and carry out all experiments with 3 random seeds and report mean and std.\ deviation. 10\% of the training data is hold out as validation set for model selection.
The hyperparameters are listed in Section~\ref{app:cnn_details}. 

Unlike the MLP setting, we find now that the choice of $\beta$ has some influence on the stability of the training method. Note that~\cite{hoier2023conf} actually reported results for $\beta=1/\text{batch-size}$ (instead of $\beta=1$) as the (linearized) \emph{mean} CE loss was used.
We instead use the summed CE loss (so the feedback signal is not scaled by the batchsize). Results for DP$\tr$ and $\beta\in\{0.01, 0.1, 1.0\}$ in Table~\ref{tab:cifar} show that the choice of $\alpha$ has no noticeable impact in the weak feedback setting ($\beta=0.01$), which is to be expected.
In the moderate feedback setting ($\beta=0.1$), DP$\tr$ with $\alpha=1$ fails, and DP$\tr$ with $\alpha=0.5$ eventually diverges on CIFAR10 after reaching an acceptable model. The choice $\beta=1.0$ requires $\alpha=0$ for successful training and leads to a drop in performance.
Fig.~\ref{fig:grad_angle} compares the angle between DP$\tr$ gradient estimates and backprop gradient estimates when $\alpha=0$ in the strong ($\beta=1$) and weak ($\beta=0.01$) feedback setting. It can be seen in Figs.~\ref{fig:grad_angle_epochs_beta001} and~\ref{fig:grad_angle_epochs_beta1}, that DP$\tr$ with $\beta=0.01$ always provides excellent gradient estimates, whereas the gradients are initially not well aligned when $\beta=1$.
The difference is especially pronounced during the first 200 batches (a fifth of an epoch) in Figs.~\ref{fig:grad_angle_batches_beta001} and~\ref{fig:grad_angle_batches_beta1}.

\textbf{Imagenet32x32}
We restrict the more compute intensive ImageNet32x32 experiments to the setting $\alpha=1/2$ and $\beta=0.01$. We use 5\% of the training data for validation and model selection, and use the public validation dataset to evaluate the selected model. With top-1 accuracy of $41.59\pm0.11$ and top-5 accuracy of $65.63\pm0.19$ it is apparent that the repeated inference scheme maintains performance on par with (or slightly above) the results reported for DP and BP in~\citep{hoier2023conf}.

\section{Discussion}
\label{sec:bioplausibility}

\paragraph{Biological plausibility of the dyadic neuron model}
In the DP framework neurons have two internal compartments, which integrate both bottom-up and top-down input signals. 
Introducing distinct neural compartments is somewhat biologically plausible as recent neuroscience research suggests that dendrites are compartmental structures capable of performing fairly complex operations~\cite{chavlis2021bio_dendrites_2_ANNs, beniaguev2021single_cortical_column}.

It has been proposed in~\citep{Guerguiev2017,richards2019dendritic} that distal apical dendrites (integrating top-down feedback signals) might be responsible for credit assignment in pyramidal neurons, whereas basal dendrites integrate local feed-forward and recurrent inputs. The DP algorithm does not fully fit this picture as the state updates
require integrating bottom-up input and feed-back input jointly.
Further, DP relies on some form of multiplexing within a neuron in order to communicate both the neural state difference
``downstream'' and the mean state ``upstream.'' 
This distinction is illustrated in Fig.~\ref{fig:dyadic_and_pyramidal_neuron}.
In spite of these caveats, the dyadic neuron model shows that---by taking a small step away from the single compartment based McCulloch Pitts neuron model---one can nevertheless achieve asynchronous propagation of meaningful errors entirely driven by local neural activities. While still far from biological reality this may be useful for goals such as on-chip training of neuromorphic hardware.

\paragraph{Compartmental interpretation of related models}
The contrastive, lifted and difference target propagation based models can all be interpreted as compartmental models (either compartments within a neuron or within a small neural circuitry). In EP and CHL, neurons need to store the activities of the two temporally distinct inference phases. In dual propagation and in LPOM, neurons are required to maintain the nudged internal states. Neurons in difference target propagation are expected to have compartments for storing predictions, targets and possibly even noise perturbed states~\citep{lee2015}.
Works such as~\citep{Guerguiev2017,sacramento2018} explicitly focus on building biologically inspired compartmental neuron models, although these methods incur even higher computational costs by also modelling spiking behaviour. 
The segregated dendrite model~\citep{Guerguiev2017} requires two temporally distinct phases making it closer to a neuroscientists' realization of CHL or EP. The dendritic cortical microcircuit model~\citep{sacramento2018} is single phased and implementation-wise similar to fully asymmetric DP.

\paragraph{Timescales of neural dynamics}
When training a network with BP, the weights of a layer can not be updated until the forward pass has completed, and the error has at least been back-propagated to the layer in question. This strict synchronization requirement has motivated several algorithms~\cite{jaderberg2017decoupled, nokland2019training}, which---to varying degree---relax the need for precise synchronization.
One way to avoid synchronisation across layers entirely is to use auxiliary losses at each layer, which often is applied in unsupervised learning methods such as ART~\cite{grossberg1987competitive}, deep belief networks~\cite{hinton2006fast,bengio2006greedy} or---more recently---soft-Hebbian learning~\cite{journe2023hebbian}.
In the supervised learning setting, many lifted neural network approaches (such as~\cite{carreira2014distributed,whittington2017approximation,li2020}, that result in a joint minimization problem over neural states and synaptic weights) in principle rely on minimal synchronization (though in practice, inference of the neural activities is usually run to convergence to avoid storing intermediate neural states).

Among biologically inspired alternatives to BP, difference target propagation~\cite{lee2015} has similar synchronization constraints as BP, while CHL and EP require temporally distinct inference phases to converge before any weight update. DP only has a single inference phase and does not impose a strict schedule on the order in which layers are traversed, but still requires that meaningful error signals are able to reach layers before their weights are updated (see e.g. the ''random-DP'' variant in~\citep{hoier2023conf}). In a setting with asynchronous neurons this corresponds to a requirement that synaptic plasticity takes place at a (much) slower timescale than neural activities (which is at the core of continuous-time neuron models used in computational neuroscience).
This agrees with the observation that neural activity, short-term plasticity (e.g. working memory and attention) and long-term plasticity take place at entirely different timescales (10ms, 100ms-minutes and minutes-hours respectively, e.g.~\cite{tsodyks1998neural,abbott2004synaptic,citri2008synaptic}).

\paragraph{Weight transport} 
One of the key objections to BP is that it is biologically implausible for the forwards and backwards pathways to be symmetric (using $W$ to propagate forwards and $W^\top$ to propagate errors backwards). This is known as the weight transport problem~\citep{crick1989recent,grossberg1987competitive}, and variations of this issue is present in many NGRAD algorithms, including DP.
The weight transport problem is not addressed in the current work, but we note that DP has previously been shown to be compatible with Kolen-Pollack learning~\citep{kolenpollack1994} of the feedback weights~\citep{hoier2023conf}. There are a number of other interesting attempts at solving the weight transport problem in the literature such as random feedback alignment~\cite{lillicrap2014,nokland2016,frenkel2021learning} and stochastic alignment methods~\cite{Akrout2019,ernoult2022}.

\section{Conclusion}

Fully local activity difference based learning algorithms are essential for achieving on-device training of neuromorphic hardware. However, the majority of works require distinct inference phases and are costly to simulate. The dual propagation formulation overcomes these issues but also relies on symmetric nudging (at each layer), which may itself be too strict a requirement in noisy hardware. Furthermore, a rigorous derivation of the DP objective has so far been missing.
In this work we first present the theoretical foundation for dual propagation and argue (in theory and practice) why asymmetric nudging can be beneficial, despite requiring slow inference methods (instead of fast fixed-point dynamics) on digital hardware.
Further, we present an improved variant of dual propagation derived from a suitable Lagrangian,
which recovers the dynamics of dual propagation in the case of symmetric nudging ($\alpha=1/2$). In the case of asymmetric nudging (such as $\alpha\in\{0,1\}$) the new variant leads to slightly modified fixed-point dynamics, which are in fact significantly more stable as demonstrated in our experiments.

\paragraph{Acknowledgements}
This work was supported by the Wallenberg AI, Autonomous Systems and Software Program (WASP) funded by the Knut and Alice Wallenberg Foundation, and by the Chalmers AI Research Centre (CHAIR).
The experiments were enabled by the supercomputing resource Berzelius provided by National Supercomputer Centre at Linköping University and the Knut and Alice Wallenberg foundation. 
We would also like to thank Maxence Ernoult for helpful discussions.

{
  \small
  \bibliographystyle{plain}
  \bibliography{main}
}

\newpage
\appendix
\onecolumn

\section{Additional Results}\label{app:results}

Fig.~\ref{fig:extra_results} illustrate additional statistics on the alignment of back-propagation induced gradients (w.r.t.\ the network parameters) and the ones obtained by the proposed DP improvement.

\begin{figure}[htb]
     \centering
     \begin{subfigure}[b]{0.49\textwidth}
         \centering
         \includegraphics[width=\textwidth]{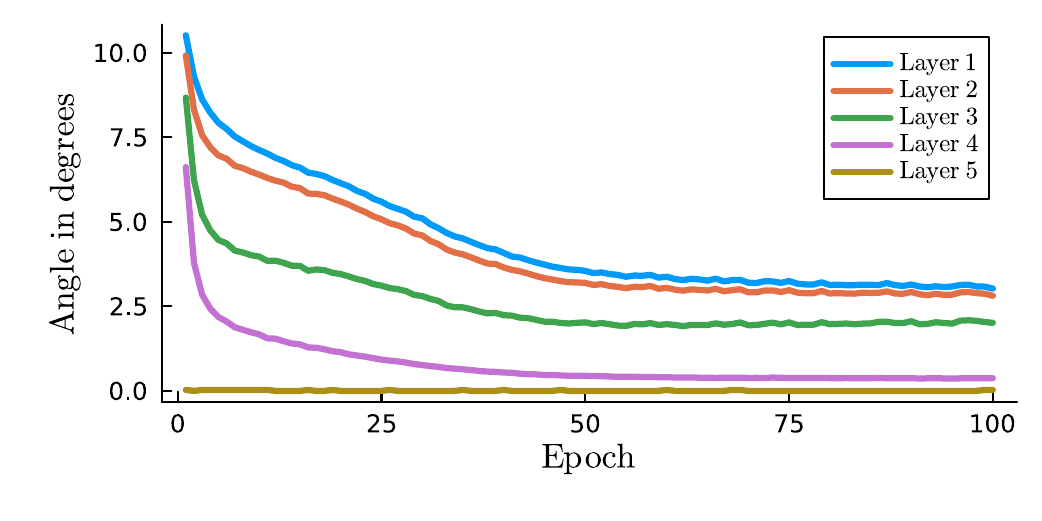}
         \caption{$\alpha=0$}
     \end{subfigure}
     \hfill
     \begin{subfigure}[b]{0.49\textwidth}
         \centering
         \includegraphics[width=\textwidth]{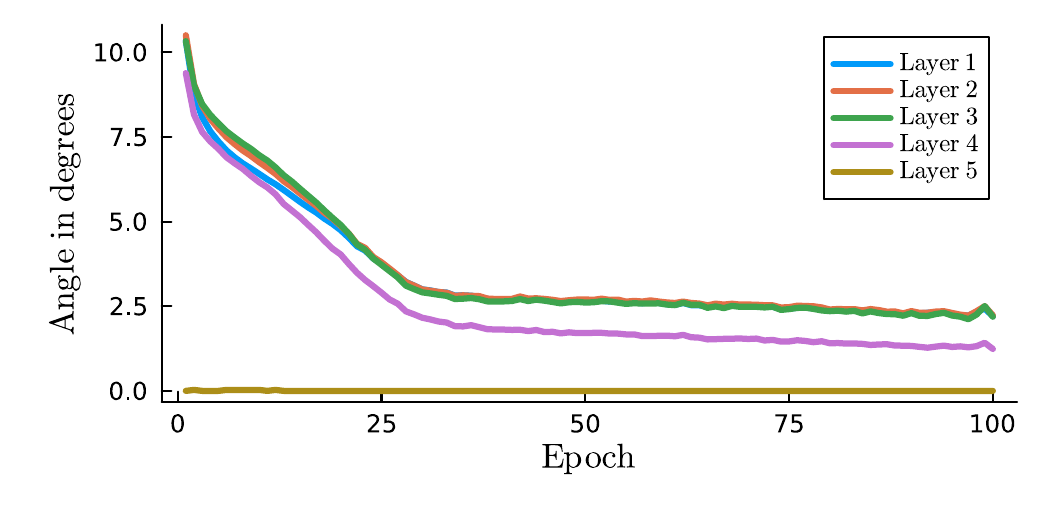}
         \caption{$\alpha=1$}
     \end{subfigure}
     \begin{subfigure}[b]{0.48\textwidth}
         \centering
         \includegraphics[width=\textwidth]{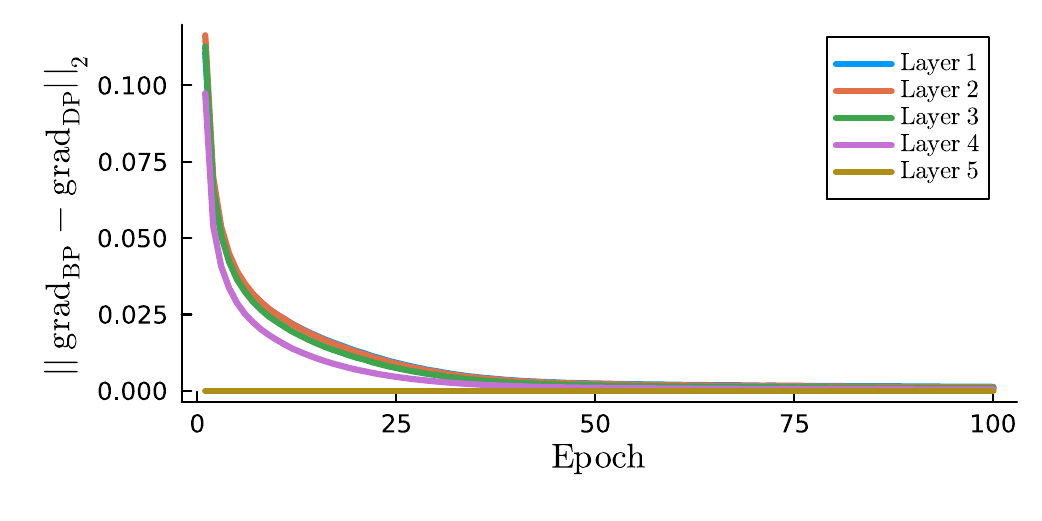}
         \caption{$\alpha=0.0$}
     \end{subfigure}
    \hfill
     \begin{subfigure}[b]{0.48\textwidth}
         \centering
         \includegraphics[width=\textwidth]{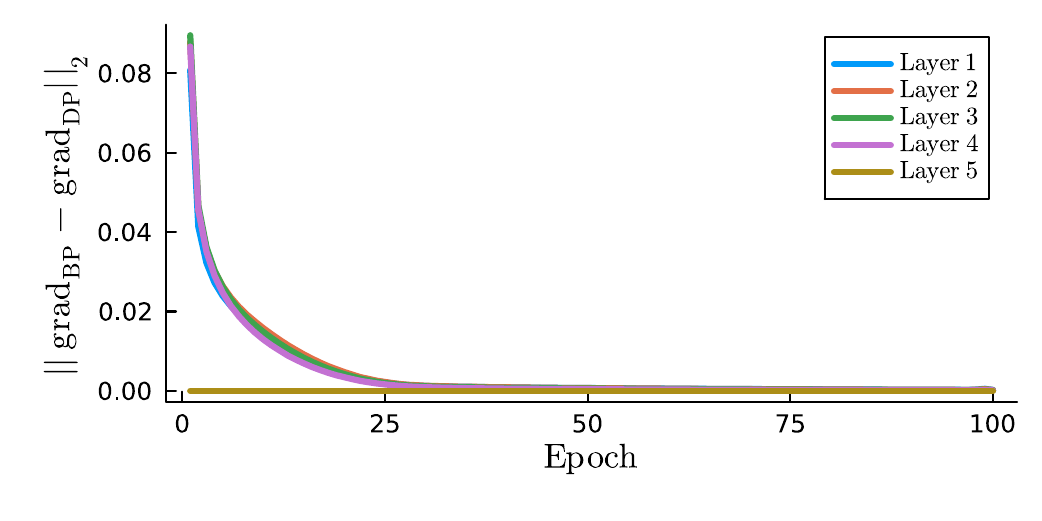}
         \caption{$\alpha=1.0$}
     \end{subfigure}
     \hfill
     \begin{subfigure}[b]{0.48\textwidth}
         \centering
         \includegraphics[width=\textwidth]{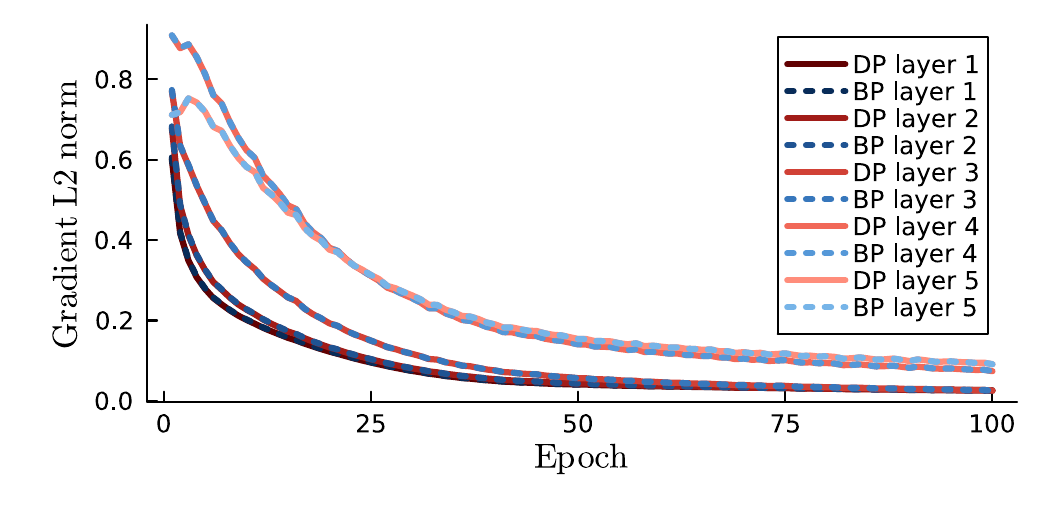}
         \caption{$\alpha=0.0$}
     \end{subfigure}
    \hfill
     \begin{subfigure}[b]{0.48\textwidth}
         \centering
         \includegraphics[width=\textwidth]{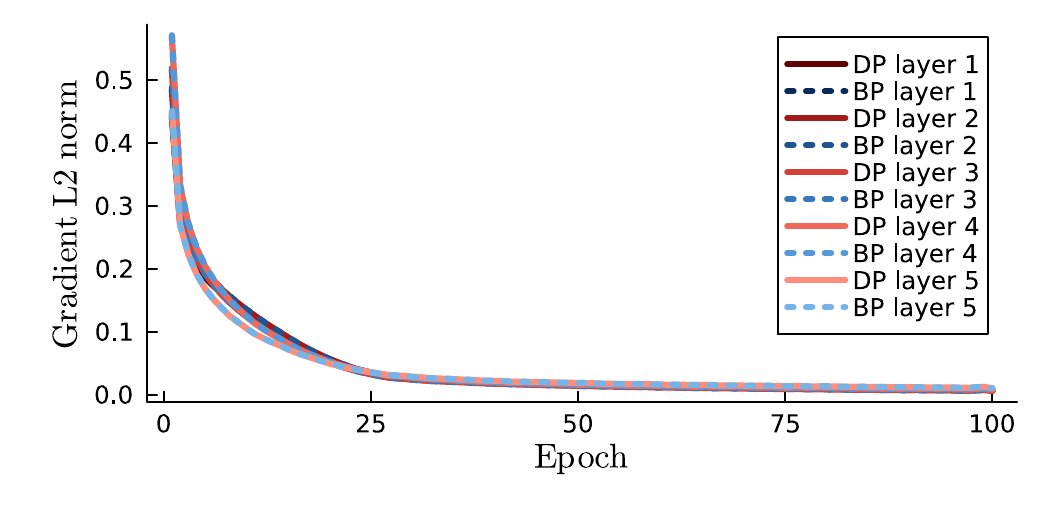}
         \caption{$\alpha=1.0$}
     \end{subfigure}
        \caption{
        MNIST experiments employing asymetric $\alpha$. Results are averaged over five random seeds.
        \textbf{Top:} Alignment between the parameter updates obtained with back-propagation and with the improved DP variant (using 30 inference iterations and asymmetric nudging with $\alpha\in\{0,1\}$). 
        \textbf{Middle:} L2 norm of difference between BP and DP gradients.
        \textbf{Bottom:} L2 norms of BP and DP gradients when using 30 inference iterations and asymmetric nudging ($\alpha=0.0$ and $\alpha=1.0$). Results are averaged over five random seeds.}
        \label{fig:extra_results}
\end{figure}

\subsection{Divergence of DP}
\label{sec:DP_failure}
The original dual propagation and the improved variant differ in the case $\alpha \in [0,1]$, $\alpha\ne 1/2$, with the boundary cases $\alpha\in \{0,1\}$ being of particular interest. The most efficient way to implement dual propagation is by sequentially updating neurons in every layer from input to output and then from output to input (akin to a forward and a backward traversal, which are nevertheless part of the sole inference phase using the same dynamics) before performing a weight update. 
However, to show the unstable behaviour of the original dual propagation formulation in the case of asymmetric nudging we instead perform repeated inference passes through the layers (which also better matches a continuous/asynchronous compute model).
Each pass (or iteration) corresponds to updating all neurons from input to output layer and back.

The results of these experiments are summarized in Tab~\ref{tab:results}, where we trained a 784-1000($\!\times\!$4)-10 MLP with ReLU activation functions on MNIST using the least-squares loss. The nudging strength $\beta$ is 1/2, which is also compatible with the original DP method.
We observe that inference in the original variant of DP diverges when applying asymmetric nudging and multiple inference iterations. This is not surprising as inference for dual propagation is only guaranteed to converge in the symmetric setting. The new variant of DP on the other hand yields stable inference results in all cases.

\section{CNN experimental details}\label{app:cnn_details}
The experiments of Section~\ref{sec:vgg16} were carried out with a VGG16 network and the following hyper parameters:
\begin{table}[h]
\small
\centering
\caption{Hyper parameters}
\begin{tabular}{l|ccc|ccc|c|c|c}
\hline
 & \multicolumn{3}{c|}{Epochs} & \multicolumn{3}{c|}{Learning rate} & \multirow{2}{*}{Momentum} & \multirow{2}{*}{Weight-decay} & \multirow{2}{*}{Batchsize} \\
Dataset  & Warmup & Decay & Total & Initial & Peak  & Final &     &      &     \\ \hline
CIFAR10  & 10            & 120          & 130   & 0.001   & 0.025 & 2e-6  & 0.9 & 5e-4 & 100 \\ \hline
CIFAR100 & 10            & 190          & 200   & 0.001   & 0.015 & 2e-6  & 0.9 & 5e-4 & 50  \\ \hline
Imagenet32x32 & 10       & 120          & 130   & 0.001   & 0.020 & 2e-6  & 0.9 & 5e-4 & 250 \\ \hline
\end{tabular}
\end{table}

\section{Proof of Proposition~\ref{prop:dual_prop}}
\begin{proof}
  The first-order optimality conditions for $s_k$ and $\delta_k$ ($k=1,\dotsc,L-1$) in Eq.~\ref{eq:Lagrangian} are given by
  \begin{align}
    \text{(I)~~~~} 0 = (f_k^{-1})'(s_k) \delta_k - W_k\tr \delta_{k+1} & & \text{(II)~~~~} 0 = f_k^{-1}(s_k) - W_{k-1} s_{k-1}.
  \end{align}
  Reparametrization in terms of $s_k^\pm$ (Eq.~\ref{eq:reparametrization}) and expanding $\text{(II)}+\alpha\text{(I)}$ and $\text{(II)}-\bar\alpha\text{(I)}$ yields
\begin{align}
\begin{split}
  0 &= \;\;\,f_k^{-1}(\alpha s_k^++\bar\alpha s_k^-) - a_k + \alpha (f_k^{-1})'(\bar s_k)\tr (s_k^+-s_k^-) - \alpha W_k\tr (s_{k+1}^+-s_{k+1}^-) \\
  0 &= -f_k^{-1}(\alpha s_k^++\bar\alpha s_k^-) + a_k + \bar\alpha (f_k^{-1})'(\bar s_k)\tr (s_k^+-s_k^-) - \bar\alpha W_k\tr (s_{k+1}^+-s_{k+1}^-),
\end{split}
\label{eq:zk_optimality_cond}
\end{align}
which are also the KKT conditions for $\cL^{DP+}_\alpha$ in Eq.~\ref{eq:new_Lagrangian}.
It remains to manipulate these to obtain the desired update equations.
Adding the equations above results in
\begin{align}
  0 &= (f_k^{-1})'(\bar s_k)\tr\!(s_k^+-s_k^-) - W_k\tr (s_{k+1}^+-s_{k+1}^-) .
\end{align}
Dual propagation is (via its use of finite differences) intrinsically linked to a (locally) linear assumption on $f_k$. Hence, we assume
    $f_k(a) = s_k^0 + D_k a + O(\norm{a-a_k}^2)$
with $s_k^0 = f_k(a_k) - D_k a_k$ and $D_k = f_k'(a_k)$. The local linear assumption allows us to neglect the higher order terms.
Consequently we also assume that $f_k^{-1}$ is locally linear and therefore $(f_k^{-1})'(\bar s_k) \approx D_k^{-1}$ is independent of $\bar s_k$. Hence, we arrive at
\begin{align}
  0 \approx D_k^{-\top} (s_k^+-s_k^-) - W_k\tr (s_{k+1}^+-s_{k+1}^-) \iff s_k^+ - s_k^- \approx D_k\tr W_k\tr (s_{k+1}^+-s_{k+1}^-).
\end{align}
We insert this into the second of the necessary optimality conditions~\eqref{eq:zk_optimality_cond} and obtain
\begin{align}
\begin{split}
    0 &= f_k^{-1}(\alpha s_k^++\bar\alpha s_k^-) - a_k = f_k^{-1}\big( s_k^- + \alpha (s_k^+-s_k^-) \big) - a_k \\
    &= f_k^{-1}\!\Paren{ s_k^- + \alpha D_k\tr W_k\tr (s_{k+1}^+-s_{k+1}^-) } - a_k
    = D_k^{-1}\!\Paren{ s_k^- - s_k^0 + \alpha D_k\tr W_k\tr (s_{k+1}^+-s_{k+1}^-) } - a_k.
\end{split} \nonumber
\end{align}
The last line is equivalent to
\begin{align}
\begin{split}
    s_k^- &= D_k a_k + s_k^0 - \alpha D_k\tr W_k\tr (s_{k+1}^+-s_{k+1}^-) 
    \approx f_k\!\Paren{ W_{k-1} \bar s_{k-1} - \alpha W_k\tr (s_{k+1}^+-s_{k+1}^-) }.
\end{split}
\label{eq:move_inside_fk}
\end{align}
Analogously, $s_k^+ \approx f_k(W_{k-1} \bar s_{k-1} + \bar\alpha W_k\tr (s_{k+1}^+-s_{k+1}^-))$.
In summary we obtain the update rules shown in~\eqref{eq:new_zk_updates}.
\end{proof}

\section{Proof of Proposition~\ref{thm:adv_CHL}}

For convenience we restate Prop.~\ref{thm:adv_CHL}:
\begin{proposition}
  Let $0 < \beta \le \beta'$ and $\alpha, \alpha'\in[0,1]$ with
  $\alpha\le\alpha'$. Then the following holds: (a)
  $J_{\alpha',\beta}^{\text{AROVR}}(\theta) \le
  J_{\alpha,\beta}^{\text{AROVR}}(\theta)$ and (b)
  $\beta J_{\alpha,\beta}^{\text{AROVR}}(\theta) \le \beta'
  J_{\alpha,\beta'}^{\text{AROVR}}(\theta)$.
\end{proposition}

This is a consequence of the following lemma.
\begin{lemma}
     Let
\begin{align}
  s_1 = \arg\min_s \alpha_1 f(z) + g(z) & & s_2 = \arg\min_s \alpha_2 f(z) + g(z)
\end{align}
for $\alpha_2 > \alpha_1$. Then $f(s_2) \le f(s_1)$. 
\end{lemma}
\begin{proof}
    Optimality of $s_{1,2}$ implies
\begin{align}
  \alpha_1 f(s_1) + g(s_1) &\le \alpha_1 f(s_2) + g(s_2) \nonumber \\
  \alpha_2 f(s_2) + g(s_2) &\le \alpha_2 f(s_1) + g(s_1) \nonumber.
\end{align}
Adding these inequalities yields
\begin{align}
  \alpha_1 f(s_1) + \alpha_2 f(s_2) \le \alpha_1 f(s_2) + \alpha_2 f(s_1)
  \iff (\alpha_2-\alpha_1) f(s_2) \le (\alpha_2-\alpha_1) f(s_1) ,
\end{align}
i.e.\ $f(s_2) \le f(s_1)$ since $\alpha_2-\alpha_1>0$.
\end{proof}

\begin{proof}[Proof of the proposition]
{\bf Claim (a)}
    In the following we absorb $1/\beta$ into $E$ and omit the dependency of $E$ on $\theta$ for notational brevity. $\alpha\le\alpha'$ implies $-\bar\alpha \le -\bar\alpha' \le 0 \le \alpha \le \alpha'$. Further, we define
    \begin{align}
    \begin{split}
        s_\alpha &:= \arg\min_s \alpha\ell(s) + E(s) \qquad s_{-\bar\alpha} := \arg\min_s -\bar\alpha\ell(s) + E(s) \\
        s_{\alpha'} &:= \arg\min_s \alpha'\ell(s) + E(s) \qquad s_{-\bar\alpha'} := \arg\min_s -\bar\alpha'\ell(s) + E(s).
    \end{split}
    \end{align}
    We identify $f(s) = \ell(s)$ and $g(s) = E(s,\theta)$ and apply the lemma repeatedly to deduce that
    \begin{align}
        \ell(s_{\alpha'}) \le \ell(s_{\alpha}) \le \ell(s_{-\bar\alpha'}) \le \ell(s_{-\bar\alpha}).
    \end{align}
    Now
    \begin{align}
        \begin{split}
            J_{\alpha',\beta}^{\text{AROVR}}(\theta) &= \underbrace{\alpha' \ell(s_{\alpha'}) + E(s_{\alpha'})}_{\le \alpha' \ell(s_{\alpha}) + E(s_{\alpha})}
            + \bar\alpha' \ell(s_{-\bar\alpha'}) - E(s_{-\bar\alpha'}) \\
            &\le \alpha' \ell(s_{\alpha}) + E(s_{\alpha}) + \bar\alpha' \ell(s_{-\bar\alpha'}) - E(s_{-\bar\alpha'}) \\
            &= \alpha \ell(s_{\alpha}) + E(s_{\alpha}) + \bar\alpha \ell(s_{-\bar\alpha'}) - E(s_{-\bar\alpha'})
            + (\alpha'-\alpha)\ell(s_{\alpha}) + (\bar\alpha'-\bar\alpha) \ell(s_{-\bar\alpha'}) \\
            &= \alpha \ell(s_{\alpha}) + E(s_{\alpha}) + \bar\alpha \ell(s_{-\bar\alpha'}) - E(s_{-\bar\alpha'})
            + (\alpha'-\alpha) \underbrace{ \big( \ell(s_{\alpha}) - \ell(s_{-\bar\alpha'}) \big) }_{\le 0} \\
            &\le \alpha \ell(s_{\alpha}) + E(s_{\alpha}) + \underbrace{\bar\alpha \ell(s_{-\bar\alpha'}) - E(s_{-\bar\alpha'})}_{\le \bar\alpha \ell(s_{-\bar\alpha}) - E(s_{-\bar\alpha})} \\
            &\le \alpha \ell(s_{\alpha}) + E(s_{\alpha}) + \bar\alpha \ell(s_{-\bar\alpha}) - E(s_{-\bar\alpha}) \\
            &= J_{\alpha,\beta}^{\text{AROVR}}(\theta),
        \end{split}
    \end{align}
    where we used the definition of $s_{\alpha}$ in the first inequality, $\ell(s_{\alpha})\le \ell(s_{-\bar\alpha'})$ (together with $\alpha'-\alpha\ge 0$) for the second one, and finally the definition of $s_{-\bar\alpha}$ to obtain the last inequality.

{\bf Claim (b)}
We proceed similar to above and define the non-negative quantities
\begin{align}
  \gamma := \alpha\beta & & \bar\gamma := \bar\alpha\beta & & \gamma' := \alpha\beta' & & \bar\gamma' := \bar\alpha\beta'.
\end{align}
The assumption $\beta\le\beta'$ implies
\begin{align}
  -\bar\gamma' \le -\bar\gamma \le 0 \le \gamma \le \gamma'.
\end{align}
We also introduce
\begin{align}
  \begin{split}
    s_\gamma &:= \arg\min_s \gamma\ell(s) + E(s) \qquad s_{-\bar\gamma} := \arg\min_s -\bar\gamma\ell(s) + E(s) \\
    s_{\gamma'} &:= \arg\min_s \gamma'\ell(s) + E(s) \qquad s_{-\bar\gamma'} := \arg\min_s -\bar\gamma'\ell(s) + E(s).
  \end{split}
\end{align}
From the lemma we conclude that
\begin{align}
  \ell(s_{\gamma'}) \le \ell(s_{\gamma}) \le \ell(s_{-\bar\gamma}) \le \ell(s_{-\bar\gamma'}).
\end{align}
\begin{align}
  \begin{split}
    \beta J_{\alpha,\beta}^{\text{AROVR}}(\theta) &= \gamma \ell(s_\gamma) + E(s_\gamma) + \bar\gamma \ell(s_{-\bar\gamma}) - E(s_{-\bar\gamma}) \\
    &\le \gamma \ell(s_{\gamma'}) + E(s_{\gamma'}) + \bar\gamma \ell(s_{-\bar\gamma}) - E(s_{-\bar\gamma}) \\
    &= \gamma' \ell(s_{\gamma'}) + E(s_{\gamma'}) + \bar\gamma' \ell(s_{-\bar\gamma}) - E(s_{-\bar\gamma})
    + (\gamma-\gamma') \ell(s_{\gamma'}) + (\bar\gamma-\bar\gamma') \ell(s_{-\bar\gamma})\\
    &= \gamma' \ell(s_{\gamma'}) + E(s_{\gamma'}) + \bar\gamma' \ell(s_{-\bar\gamma}) - E(s_{-\bar\gamma})
    + (\gamma-\gamma') \underbrace{\big( \ell(s_{\gamma'}) - \ell(s_{-\bar\gamma}) \big)}_{\le 0} \\
    &\le \gamma' \ell(s_{\gamma'}) + E(s_{\gamma'}) + \bar\gamma' \ell(s_{-\bar\gamma}) - E(s_{-\bar\gamma}) \\
    &\le \gamma' \ell(s_{\gamma'}) + E(s_{\gamma'}) + \bar\gamma' \ell(s_{-\bar\gamma'}) - E(s_{-\bar\gamma'}) \\
    &= \beta' J_{\alpha,\beta'}^{\text{AROVR}}(\theta).
  \end{split}
\end{align}
We applied the definition of $s_\gamma$ (first inequality),
$\ell(s_{\gamma'}) \le \ell(s_{-\bar\gamma})$ (2nd inequality) and the
definition of $s_{-\bar\gamma'}$.
\end{proof}

\section{Proof of Proposition~\ref{thm:adv_CHL_2}}

For convenience we restate Prop.~\ref{thm:adv_CHL_2}:
\begin{proposition}
  Let $0 < \beta' < \beta$ and $\alpha'$ such that
  $\bar\alpha\beta=\bar\alpha'\beta'$. Then
  $J_{\alpha,\beta}^{\text{AROVR}}(\theta) \le
  J_{\alpha',\beta'}^{\text{AROVR}}(\theta)$.
\end{proposition}
\begin{proof}
  Let $\gamma := \bar\alpha\beta = \bar\alpha'\beta'$. We further introduce
  \begin{align}
    \tilde E_{-\gamma}(s) := -\gamma\ell(s) + E(s)
  \end{align}
  and
  \begin{align}
    \begin{split}
    s_{\alpha\beta} &= \arg\min_s \alpha\ell(s) + \tfrac{1}{\beta} E(s) \\
    s_{\alpha'\beta'} &= \arg\min_s \alpha'\ell(s) + \tfrac{1}{\beta'} E(s)  \\
    s_{-\gamma} &= \arg\min_s -\bar\alpha\ell(s) + \tfrac{1}{\beta} E(s) = \arg\min_s \tilde E_{-\gamma}(s) .
    \end{split}
  \end{align}
  Now
  \begin{align}
    \begin{split}
      J_{\alpha,\beta}^{\text{AROVR}}(\theta) &= \alpha\ell(s_{\alpha\beta}) + \tfrac{1}{\beta} E(s_{\alpha\beta}) +
      \bar\alpha\ell(s_{-\gamma}) - \tfrac{1}{\beta} E(s_{-\gamma}) \\
      &\le \alpha\ell(s_{\alpha'\beta'}) + \tfrac{1}{\beta} E(s_{\alpha\beta}) + \bar\alpha\ell(s_{-\gamma}) - \tfrac{1}{\beta} E(s_{-\gamma}) \\
      &= \ell(s_{\alpha'\beta'}) - \bar\alpha \ell(s_{\alpha'\beta'}) + \tfrac{1}{\beta} E(s_{\alpha\beta}) +
      \bar\alpha\ell(s_{-\gamma}) - \tfrac{1}{\beta} E(s_{-\gamma}) \\
      &= \ell(s_{\alpha'\beta'}) + \tfrac{1}{\beta} \tilde E_{-\gamma}(s_{\alpha'\beta'}) - \tfrac{1}{\beta} \tilde E_{-\gamma}(s_{-\gamma}) \\
      &= \ell(s_{\alpha'\beta'}) + \tfrac{1}{\beta} \underbrace{\Paren{ \tilde E_{-\gamma}(s_{\alpha'\beta'}) - \tilde E_{-\gamma}(s_{-\gamma}) } }_{\ge 0} \\
      &\le \ell(s_{\alpha'\beta'}) + \tfrac{1}{\beta'} \Paren{ \tilde E_{-\gamma}(s_{\alpha'\beta'}) - \tilde E_{-\gamma}(s_{-\gamma}) } \\
      &= \alpha'\ell(s_{\alpha'\beta'}) + \tfrac{1}{\beta'} E(s_{\alpha'\beta'}) +
      \bar\alpha'\ell(s_{-\gamma}) - \tfrac{1}{\beta'} E(s_{-\gamma}) = J_{\alpha',\beta'}^{\text{AROVR}}(\theta),
    \end{split}
  \end{align}
  where we used the minimality of $s_{\alpha\beta}$ and $s_{-\gamma}$, and
  that $1/\beta < 1/\beta'$ to obtain the inequalities.
\end{proof}

\section{Stabilized Fixed-Point Iterations}
\label{sec:dampened_fixed_point}

The LPOM model considered in~\cite{li2020} uses specific layerwise potentials $E_k$ and also layer-dependent values of $\beta$,
\begin{align}
  \cL^{\text{LPOM}}(\theta) = \min_{\{s_k\}} \ell(s_L) + \sum_k \tfrac{1}{\beta_k} \Paren{ G_k(s_k) - s_k\tr W_{k-1}s_{k-1} + G_k^*(W_{k-1}s_{k-1}) },
\end{align}
where we use $s_k$ instead of $s_k^+$ as the $s_k^-$ are already eliminated. The LPOM authors suggest fixed-point iterations of the form
\begin{align}
  s_k^{(t)} \gets f_k \Paren{ W_{k-1} s_{k-1} + \tfrac{\beta_k}{\beta_{k+1}} W_k\tr \Paren{ s_{k+1} - f_{k+1}(W_ks_k^{(t-1)}) } } ,
  \label{eq:LPOM_updates}
\end{align}
which can be derived from the stationarity condition w.r.t.\ $s_k$,
\begin{align}
  \begin{split}
    0 &= \tfrac{1}{\beta_k} \Paren{ \nabla G_k(s_k) - W_{k-1}s_{k-1} } + \tfrac{1}{\beta_{k+1}} \Paren{ W_k\tr \nabla G_{k+1}^*(W_ks_k) - W_k\tr s_{k+1} } \\
    &= \tfrac{1}{\beta_k} \Paren{ f_k^{-1}(s_k) - W_{k-1}s_{k-1} } + \tfrac{1}{\beta_{k+1}} W_k\tr \!\Paren{ f_{k+1} (W_ks_k) - s_{k+1} }.
  \end{split}
\end{align}
In~\cite{li2020} it is assumed that each $f_k$ is Lipschitz-continuous
(w.l.o.g.\ 1-Lipschitz continuous) and that $\{\beta_k\}_k$ are chosen such that
\begin{align}
  \tfrac{\beta_k}{\beta_{k+1}} \norm{W_k\tr W_K}_2 < 1,
\end{align}
i.e.\ $\beta_k \ll \beta_{k+1}$ and thereby introducing discounting of later
layers (and enabling only weak feedback from the target loss $\ell$). In our
setting we choose $\beta_0=\cdots=\beta_{L-1}=1$ and $\beta_L=\beta$. Hence,
the scheme in Eq.~\ref{eq:LPOM_updates} is only guaranteed to converge as long
as $\norm{W_k\tr W_k}_2 < 1$. For many standard activation functions (in
particular ReLU, leaky ReLU and hard sigmoid) the underlying $G_k$ is of the
form $G_k(s_k) = \norm{s_k}^2/2 + \imath_C(s_k)$ for a convex set $C$, where
$\imath_C(s_k)$ is the functional form of the constraint $s_k\in C$. Therefore
it is convenient to add a quadratic dampening term when optimizing w.r.t.\
$s_k$. Let
\begin{align}
  \Phi_k(s_k;s_{k-1},s_{k+1},\theta) &:= G_k(s_k) - s_k\tr W_{k-1}s_{k-1} + G_{k+1}^*(W_ks_k) - s_{k+1}\tr W_ks_k \\
  \tilde\Phi_k (s_k;s_k^0,s_{k-1},s_{k+1},\theta) &:= G_k(s_k) - s_k\tr W_{k-1}s_{k-1} + (s_k-s_k^0)\tr W_k\tr f_{k+1}(W_ks_k^0) - s_{k+1}\tr W_ks_k
\end{align}
be the terms in $\cL^{\text{LPOM}}$ dependent on $s_k$ and its (partially)
linearized surrogate, respectively. Now consider the iteration
\begin{align}
  \label{eq:dampened_iteration}
  \begin{split}
    s_k^{(t+1)} &:= \arg\min_{s_k} \tilde\Phi_k(s_k;s_k^{(t)},s_{k-1},s_{k+1},\theta) + \tfrac{L}{2} \norm{s_k-s_k^0}^2\\
    &= \Pi_C\!\Paren{ \frac{W_{k-1}s_{k-1} + W_k\tr \Paren{ s_{k+1} - f_{k+1}(W_ks_k^{(t-1)}) } + Ls_k^{(t)}}{1+L} }.
  \end{split}
\end{align}
If $s_k^{(t)}$ is a stationary point of $\tilde\Phi_k$, then
$s_k^{(t+1)}=s_k^{(t)}$ is also a stationary point of $\Phi_k$. If
$G_{k+1}$ is 1-strongly convex (and therefore $\nabla G_{k+1}=f_{k+1}$ is
1-Lipschitz continuous) and $L \ge \norm{W_k\tr W_k}_2$, then
$\tilde\Phi_k(\cdot; s_k^0,s_{k-1},s_{k+1},\theta) + \frac{L}{2}
\norm{\cdot-s_k^0}^2$ is a majorizer of $\Phi_k$, and the sequence
$(s_k^{(t)})_t$ decreases monotonically $\Phi_k$. It can be further shown that
Eq.~\ref{eq:dampened_iteration} is a contraction if
$L > \max(0, \norm{W\tr W}_2 - 1)/2$.

We use the stabilized fixed-point scheme in Eq.~\ref{eq:dampened_iteration}
determine the neural states $s_k$ in the setting $\alpha\in\{0,1\}$ (i.e.\ the
LPOM and inverted LPOM formulations), where we assign $L$ to the estimate of
$\norm{W_k\tr W_k}_2$ obtained by 5~power iterations.

\section{Analysis of the fixed point iterations}
\label{sec:fixed_point_analysis}

In the following we use an idealized setup, where we make the following assumptions:
\begin{enumerate}
\item We consider a trilevel program with a hidden and an output layer,
\item and we use second-order Taylor expansion for the relevant mappings.
\end{enumerate}
The trilevel program is given by
\begin{align}
  \ell(y^*) \qquad \text{s.t. } y^* = \arg\min_y F(y) - y\tr Wz^* \qquad z^* = \arg\min_z G(z) - z\tr x
  \label{eq:trilevel}
\end{align}
and $\ell$, $F$ and $G$ are given by
\begin{align}
  \label{eq:taylor}
  \ell(y) = \tfrac{1}{2} y\tr H_\ell y + b_\ell\tr y & & F(y) = \tfrac{1}{2} y\tr H_F y + b_F\tr y
  & & G(z) = \tfrac{1}{2} z\tr H_G z + b_G\tr z.
\end{align}
In the following subsection we derive the closed form expression for the
update of $z^\pm$ in case of the AROVR and SPROVR relaxations.

\subsection{Fixed-point updates based on AROVR}

Applying the AROVR step on the outermost level in Eq.~\ref{eq:trilevel} yields
\begin{align}
  U(y^+,y^-,z^+,z^-) = \alpha\ell(y^+) + \bar\alpha\ell(y^-) + F(y^+) - F(y^-) - (y^+\!-\!y^-)\tr W\bar z + G(z^+) - G(z^-) - (z^+\!-\!z^-)\tr x.
\end{align}
$\bar z=\alpha' z^+ + \bar\alpha' z^-$ uses a potentially different
coefficient $\alpha'$.  By using the assumption in Eq.~\ref{eq:taylor},
$y^\pm$ are given by
\begin{align}
  \begin{split}
    0 &= \alpha (H_\ell y^+ + b_\ell) + H_F y^+ + b_F - W z  \quad\implies  y^+ = (H_F + \alpha H_\ell)^{-1} (Wz - b_F - \alpha b_\ell) \\
    0 &= -\bar\alpha (H_\ell y^- + b_\ell) + H_F y^- + b_F - W z  \;\implies  y^- = (H_F - \bar\alpha H_\ell)^{-1} (Wz - b_F + \bar\alpha b_\ell).
  \end{split}
\end{align}
Consequently,
\begin{align}
  \begin{split}
  y^+-y^- &= (H_F + \alpha H_\ell)^{-1} (Wz - b_F - \alpha b_\ell) - (H_F - \bar\alpha H_\ell)^{-1} (Wz - b_F + \bar\alpha b_\ell) \\
  {} &= \Paren{ (H_F + \alpha H_\ell)^{-1} - (H_F - \bar\alpha H_\ell)^{-1} } (Wz - b_F)
       - \Paren{ \alpha (H_F + \alpha H_\ell)^{-1} - \bar\alpha (H_F - \bar\alpha H_\ell)^{-1} } b_\ell .
  \end{split}
\end{align}
Only the term involving $Wz$ is of interest as the remaining ones are
independent of $z$. Now define
\begin{align}
  M := (H_F + \alpha H_\ell)^{-1} - (H_F - \bar\alpha H_\ell)^{-1} ,
\end{align}
then
\begin{align}
  \begin{split}
    & (H_F + \alpha H_\ell) M = \mI - (H_F + \alpha H_\ell) (H_F - \bar\alpha H_\ell)^{-1} \\
    \iff & (H_F + \alpha H_\ell) M (H_F - \bar\alpha H_\ell) = (H_F - \bar\alpha H_\ell) - (H_F + \alpha H_\ell) = -H_\ell .
  \end{split}
\end{align}
or
\begin{align}
  M = -(H_F + \alpha H_\ell)^{-1} H_\ell (H_F - \bar\alpha H_\ell)^{-1} .
\end{align}
A similar calculation shows that
\begin{align}
  M' &:= \alpha (H_F + \alpha H_\ell)^{-1} - \bar\alpha (H_F - \bar\alpha H_\ell)^{-1} \nonumber \\
     &= (H_F + \alpha H_\ell)^{-1} \big( (1-2\alpha) H_F - 2\alpha\bar\alpha H_\ell \big) (H_F - \bar\alpha H_\ell)^{-1} \nonumber.
\end{align}
Finally,
\begin{align}
  \begin{split}
  y^+-y^- &= M (Wz - b_F) - M' b_\ell \\
  {} &= (H_F + \alpha H_\ell)^{-1} \Paren{ H_\ell (H_F - \bar\alpha H_\ell)^{-1} (b_F - Wz)
       - \big( (1-2\alpha) H_F - 2\alpha\bar\alpha H_\ell \big) (H_F - \bar\alpha H_\ell)^{-1} b_\ell } .
  \end{split}
\end{align}
With $H_F=\mI$ and $H_\ell=\beta \mI$ we obtain
\begin{align}
  M = -\frac{\beta}{(1+\alpha\beta)(1-\bar\alpha\beta)} \,\mI .
  \label{eq:M_example_AROVR}
\end{align}
Using the quadratic approximation for $G$,
$G(z) = \frac{1}{2} z\tr H_G z + b_G\tr z$, then
\begin{align}
  \begin{split}
  z^+ &\gets H_G^{-1}\Paren{ x - b_G + \alpha'W\tr (y^+-y^-) } = g\!\Paren{ x + \alpha'W\tr (y^+-y^-) } \\
  z^- &\gets H_G^{-1}\Paren{ x - b_G - \bar\alpha'W\tr (y^+-y^-) } = g\!\Paren{ x - \bar\alpha'W\tr (y^+-y^-) } 
  \end{split}
\end{align}
and
\begin{align}
  \begin{split}
  \alpha'z^+ + \bar\alpha'z^- &\gets H_G^{-1}\Paren{ x - b_G + \Paren{ (\alpha')^2 - (\bar\alpha')^2 } W\tr (y^+-y^-) }  \\
  {} &= H_G^{-1}\Paren{ x - b_G + (2\alpha' - 1) W\tr (y^+-y^-) } \\
  {} &= g\!\Paren{ x + (2\alpha' - 1) W\tr (y^+-y^-) } \\
  {} &= g\!\Paren{ x + (2\alpha' - 1) W\tr (M W \bar z + v) }
  \end{split}
\end{align}
for a vector $v$ containing all terms in $y^+-y^-$ not dependent on
$\bar z$. The choice $\alpha'=1/2$ therefore implies that $\bar z$ is a fixed
point of the above relation.

\subsection{Fixed-point updates based on SPROVR}

The above reasoning applies to the AROVR formulation. The SPROVR model uses a different objective,
\begin{align}
  U(y^+,y^-,z^+,z^-) = \ell(\alpha y^++\bar\alpha y^-) + F(y^+) - F(y^-) - (y^+-y^-)\tr W\bar z + G(z^+) - G(z^-) - (z^+-z^-)\tr x
\end{align}
with corresponding stationary conditions
\begin{align}
  \alpha\ell'(\bar y) + \nabla F(y^+) - W\bar z = 0 & & \bar\alpha\ell'(\bar y) - \nabla F(y^-) + W\bar z = 0 .
\end{align}
By using a quadratic model for $\ell$ and $F$, this translates to
\begin{align}
  (I) \quad \alpha \Paren{ H_\ell \bar y + b_\ell } + H_F y^+ + b_F - W\bar z = 0
  & & (II) \quad \bar\alpha \Paren{ H_\ell \bar y + b_\ell } - H_F y^- - b_F + W\bar z = 0.
\end{align}
Adding these relation yields
\begin{align}
  H_\ell \bar y + b_\ell + H_F(y^+-y^-) = 0 \iff y^+-y^- = -H_F^{-1} \Paren{ H_\ell \bar y + b_\ell }.
\end{align}
Further, $\alpha (I) - \bar\alpha (II)$ results in
\begin{align}
  \begin{split}
  0 &= (\alpha^2-\bar\alpha^2) H_\ell \bar y + (\alpha^2-\bar\alpha^2) b_\ell + H_F \bar y + b_F - W\bar z  \\
    &= \Paren{ (2\alpha-1) H_\ell + H_F} \bar y + (2\alpha-1) b_\ell + b_F - W\bar z  \\
  \iff \bar y &= \Paren{ (2\alpha-1) H_\ell + H_F}^{-1} \Paren{ W\bar z - (2\alpha-1) b_\ell - b_F }
  \end{split}
\end{align}
and
\begin{align}
  y^+-y^- &= -H_F^{-1} \Paren{ H_\ell \Paren{ (2\alpha-1) H_\ell + H_F}^{-1} \Paren{ W\bar z - (2\alpha-1) b_\ell - b_F } + b_\ell }.
\end{align}
Let $M$ be the matrix applied on $W\bar z$, i.e.
\begin{align}
  M = -H_F^{-1} H_\ell \Paren{ (2\alpha-1) H_\ell + H_F}^{-1} ,
\end{align}
then for $H_F=\mI$ and $H_\ell=\beta \mI$ we obtain
\begin{align}
  M = -\frac{\beta}{1+(2\alpha-1)\beta} \,\mI = -\frac{\beta}{1 - \beta + 2\alpha\beta} \,\mI .
\end{align}
Recall the value of $M$ before (Eq.~\ref{eq:M_example_AROVR}). As
\begin{align}
  (1+\alpha\beta)(1-\bar\alpha\beta) &= 1-\beta+2\alpha\beta - \alpha\bar\alpha\beta^2 \nonumber
\end{align}
and $\alpha\bar\alpha\beta^2\ge 0$, we conclude that the SPROVR leads to a
dampened error signal compared to AROVR (which is intuitively expected).
The updates for $z^\pm$ are given by
\begin{align}
  \begin{split}
  z^+ &\gets H_G^{-1}\Paren{ x - b_G + \alpha'W\tr (y^+-y^-) } = g\!\Paren{ x + \alpha'W\tr (M W\bar z + v) } \\
  z^- &\gets H_G^{-1}\Paren{ x - b_G - \bar\alpha'W\tr (y^+-y^-) } = g\!\Paren{ x - \bar\alpha'W\tr (M W\bar z + v) } 
  \end{split}
\end{align}
for a suitable vector $v$ independent of $\bar z$. Finally,
\begin{align}
  \begin{split}
  \alpha'z^+ + \bar\alpha'z^- &\gets H_G^{-1}\Paren{ x - b_G + \Paren{ (\alpha')^2 - (\bar\alpha')^2 } W\tr (y^+-y^-) }  \\
  {} &= H_G^{-1}\Paren{ x - b_G + (2\alpha' - 1) W\tr (y^+-y^-) } \\
  {} &= g\!\Paren{ x + (2\alpha' - 1) W\tr (y^+-y^-) } \\
  {} &= g\!\Paren{ x + (2\alpha' - 1) W\tr (M W \bar z + v) },
  \end{split}
\end{align}
and we conclude that $\bar z$ is again a fixed point for the updates when
$\alpha'=1/2$.

\end{document}